\documentclass[draftcls, 12pt, a4paper]{article}

\usepackage{times}
\usepackage{graphicx} % more modern
\usepackage{subfigure} 
\usepackage{amsmath,amssymb,amsthm}

\usepackage{natbib}

\newcommand{\mbf}[1]{\mathbf{#1}}
\newcommand{\mcal}[1]{\mathcal{#1}}
\newcommand{\la}{\lambda}
\newcommand{\bs}{\backslash}

\newcommand{\kappaL}{\kappa_{\mcal{L}}}
\newcommand{\tauL}{\tau_{\mcal{L}}}
\newcommand{\kappamin}{\kappa^*_{\min}}

\newcommand{\Sighat}{\mbf{\widehat{\Sigma}}}
\newcommand{\Sigstar}{\mbf{\Sigma}^{*}}
\newcommand{\Thetastar}{\Theta^{*}}
\newcommand{\Shat}{\widehat{\mbf{S}}}
\newcommand{\Lhat}{\widehat{\mbf{L}}}
\newcommand{\Thetahat}{\widehat{\Theta}}
\newcommand{\Sstar}{\mbf{S}^*}
\newcommand{\Lstar}{\mbf{L}^*}
\newcommand{\reff}{r_{\text{eff}}}

\newcommand{\be}{\begin{equation}}
\newcommand{\ee}{\end{equation}}
\newcommand{\ALGN}[1]{\begin{align*}#1\end{align*}}
\newcommand{\ALGNN}[1]{\begin{align}#1\end{align}}

\newtheorem{thm}{Theorem}
\newtheorem{lem}{Lemma}
\newtheorem{cor}{Corollary}

\newtheorem{assumption}{Assumption}

\usepackage{enumitem}
\usepackage{authblk}

\title{Learning Latent Variable Gaussian Graphical Models}

\author[1]{Zhaoshi Meng}
\author[2]{Brian Eriksson}
\author[1]{Alfred O. Hero III}
\affil[1]{\normalsize{Department of Electrical Engineering and Computer Science, University of Michigan, Ann Arbor, MI 48109, USA}}
\affil[2]{\normalsize{Technicolor Research Center, 735 Emerson Street, Palo Alto, CA 94301, USA}}
\affil[ ]{\normalsize{mengzs@umich.edu, brian.eriksson@technicolor.com, hero@eecs.umich.edu}}
\date{}

\begin{document} 

\maketitle

\begin{abstract} 

Gaussian graphical models (GGM) have been widely used in many high-dimensional applications ranging from biological and financial data to recommender systems. Sparsity in GGM plays a central role both statistically and computationally. Unfortunately, real-world data often does not fit well to sparse graphical models.  In this paper, we focus on a family of latent variable Gaussian graphical models (LVGGM), where the model is conditionally sparse given latent variables, but marginally non-sparse.  In LVGGM, the inverse covariance matrix has a low-rank plus sparse structure, and can be learned in a regularized maximum likelihood framework. We derive novel parameter estimation error bounds for LVGGM under mild conditions in the high-dimensional setting. 
These results complement the existing theory on the structural learning, and open up new possibilities of using LVGGM for statistical inference. 
\end{abstract} 

%\section{Introduction}
%\label{sec:intro}
\section{Introduction}
\label{sec:intro}

Critical to many statistical inference tasks in complex real-world systems, such as prediction and detection, is the ability to extract and estimate distributional characteristics from the observations. Unfortunately, in the high-dimensional regime such model estimation often leads to ill-posed problems, particularly when the number of observations $n$ (or sample size) is comparable to or fewer than the ambient dimensionality $p$ of the model ({\em i.e.,} the ``large $p$, small $n$'' problem). This challenge arises in many modern real-world applications ranging from recommender systems, gene microarray data, and financial data, to name a few. To perform accurate model parameter estimation and subsequent statistical inference, low dimensional structure is often imposed for regularization~\citep{negahban2012unified}.

For Gaussian-distributed data, the central problem is often to estimate the inverse covariance matrix (alternatively known as the precision, concentration or information matrix).  Gaussian graphical models (GGM) 
%, an important class of Markov random fields, 
provide an efficient representation of the precision matrix through a graph that represents non-zeros in the matrix~\citep{lauritzen1996graphical}.  In high-dimensional regimes, this graph can be forced to be sparse, imposing a low-dimensional structure on the GGM.  For sufficiently sparse GGM, statistically consistent estimates of the model structure (\emph{i.e.,} sparsistency) can be achieved ({\em e.g.,}~\citet{ravikumar2011high}).  On the computational side, sparsity also leads to reduced complexity of the estimator~\citep{hsieh2013sparse}. 
However, when the true distribution can not be well-approximated by a sparse GGM, the standard learning paradigm suffers from either large estimation bias due to enforcing a overly sparse model, or degraded computation time for a dense model. Both result in suboptimal performance in the subsequent inference tasks. 

In this paper, we consider a new class of high-dimensional GGM for extending the standard sparse GGM. The proposed model is motivated by many real-world applications, where there exist certain exogenous and often latent factors affecting a large portion of the variables. 
Examples are the price of oil on the airlines' stock price variables~\citep{choi2010gaussian}, and the genres on movie rating variables. 
Conditioning on these \emph{global} effects, the variables are assumed to 
have highly \emph{localized} interactions, which can be well-fitted by a sparse GGM. However, due to the marginalization over global effects, the observed (marginal) GGM, and its corresponding precision matrix, is not sparse. 
Unfortunately, in this regime, existing theoretical results and computational tools for sparse GGM are not applicable. 

To address this problem, we propose to use latent variable Gaussian graphical models (LVGGM) for modeling and statistical inference. LVGGM introduce latent variables to capture the correlations due to the global effects, and the remaining effects are captured by a conditionally sparse graphical model. 
The resulting marginal precision matrix of the LVGGM has a sparse plus low-rank structure, therefore we consider a regularized maximum likelihood (ML) approach for parameter estimation (previously considered by~\citet{chandrasekaran2012latent}).  By utilizing the \emph{almost strong convexity}~\citep{kakade2009learning} of the log-likelihood, we derive a non-asymptotic parameter error bound for the regularized ML estimator. Our derived bounds apply to the high-dimensional setting of $p \gg n$ due to restricted strong convexity~\citep{negahban2012unified} and certain structural incoherence between the sparse and low-rank components of the precision matrix~\citep{yang2013dirty}. 

We show that for sufficiently large $n$, the Frobenius norm error of the precision matrix of LVGGM converges at the rate $\mcal{O}( \sqrt{\frac{(s + \reff \cdot r) \log p}{n}})$, where $s$ is the number of non-zeros in the conditionally sparse precision matrix, $\reff$ is the effective rank of the covariance matrix and $r$ is the number of latent variables.  This rate is in general significantly faster than the standard convergence rate of $\mcal{O}(\sqrt{\frac{p^2 \log{p}}{n}})$ for an unstructured dense GGM. 
This result offers a compelling argument for using LVGGM over sparse GGM for many inference problems.

%Paper structure
The paper is structured as follows. In Section~\ref{sec:relWork} we review the relevant prior literature.  In Section~\ref{sec:setup} we formulate the LVGGM estimation problem. In Section~\ref{sec:theory} the main theoretical results are presented.  Experimental results are shown in Section~\ref{sec:exp} and we conclude in Section~\ref{sec:conclude}. 
We use boldface letters to denote vectors and matrices. ${\| \cdot \|_1}$, $\| \cdot \|_2$, $\| \cdot \|_F$, $\| \cdot \|_*$ denote the elementwise $\ell_1$, spectral, Frobenius, and nuclear matrix norms, respectively.

%\section{Related Work}
%\label{sec:relWork}
\section{Background and Related Work}
\label{sec:relWork}

The problem of learning GGM with sparse inverse covariance matrices using $\ell_{1}$-regularized maximum likelihood estimation, often referred to as the graphical lasso (Glasso) problem, has been studied in~\citet{friedman2008sparse, ravikumar2011high, rothman2008sparse}. In particular,  
the authors of~\citet{ravikumar2011high} study the  model selection consistency (\emph{i.e.,} ``sparsistency'') under certain incoherence condition.  
Beyond sparse GGM, \citet{choi2010gaussian} propose a multi-resolution extension of a GGM augmented with sparse inter-level correlations, while in~\citet{choi2011learning} the authors consider latent tree-structured graphical models. Both models lead to computationally efficient inference and learning algorithms but restrict the latent structure to trees. Recently,~\citet{liu2013learning} consider a computationally efficient learning algorithm for a class of conditionally tree-structured LVGGM.  

The work that is most relevant to ours is by~\citet{chandrasekaran2012latent}, who study the LVGGM learning problem, but focus on the simultaneous model selection consistency of both the sparse and low-rank components. 
In contrast, in this paper we focus on the Frobenius norm error bounds for estimating the precision matrix of LVGGM. Although structural consistency can be useful for deriving insights, parameter estimation error analysis is of equal or greater importance in practice. Since it provides additional, and usually more direct, insights into factors influencing the performance of the subsequent statistical inference tasks, such as prediction and detection.  
Also, compared with~\citet{chandrasekaran2012latent}, our Frobenius norm error bounds are derived under mild condition on the Fisher information of the distribution. 

We note that there is a fundamentally different line of work on estimating models with a similar structural composition, known as \emph{robust PCA}~\citep{candes2011robust}. In robust PCA, the data matrix is modeled as ``low-rank plus sparse''.  This model has been applied to extracting the salient foreground from background in videos, and detecting malicious user ratings in recommender system data~\citep{xu2012robust}. In contrast, the equivalent covariance model of our LVGGM can be decomposed into a low-rank plus a dense matrix whose \emph{inverse} is sparse. 
A similar covariance model has recently been studied by~\citet{kalaitzis2012residual}, in which an EM algorithm is proposed for estimation but no theoretical error bounds are derived. In this paper, we instead focus on the precision matrix parameterization, which enables model estimation through a convex optimization. This formulation is of both theoretical and computational importance.

%\section{Problem Setup}
%\label{sec:setup}
\section{Problem Setup}
\label{sec:setup}

%------------------

\begin{figure*}[t]
\centering
%    \begin{subfigure}[Sample Covariance Matrix $\Sigma$]{0.25\textwidth}{
%        \centering
%        \includegraphics[width=0.2\textwidth]{figures/sampleCov}    
%    }
%    \end{subfigure}
    \begin{subfigure}[ ]{
        \centering
        \includegraphics[width=0.24\textwidth]{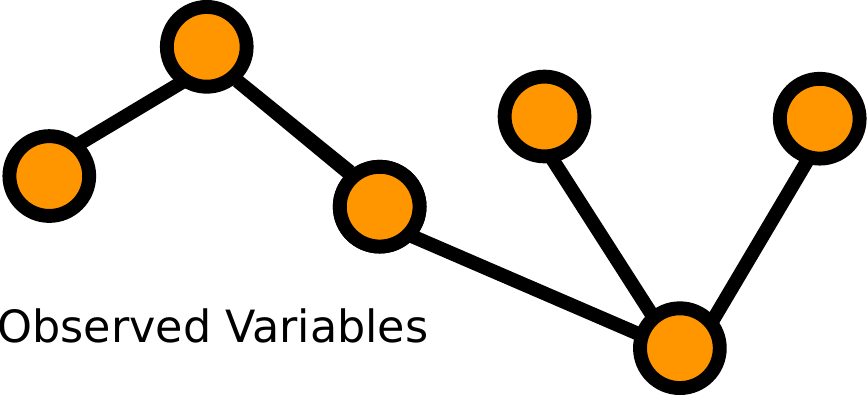}    
        \label{subfig:toy:ggm}
    }
    \end{subfigure}
    \begin{subfigure}[ ]{
        \centering
        \includegraphics[width=0.175\textwidth]{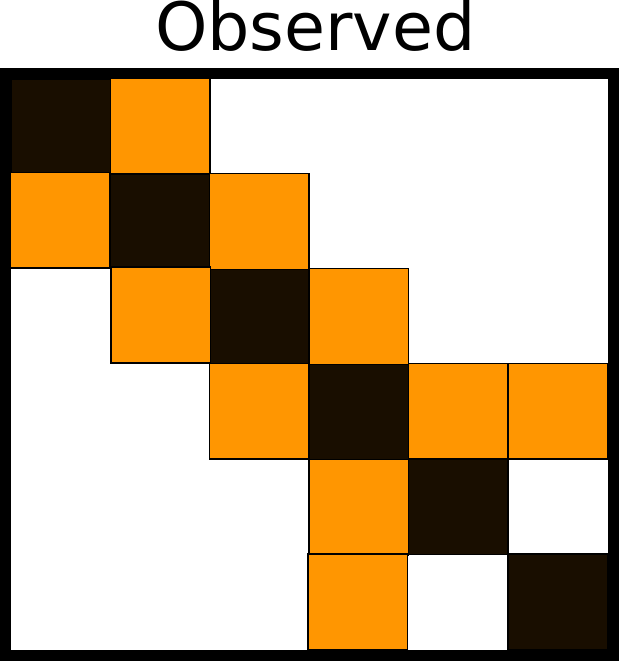}  
        \label{subfig:toy:ggmMatrix}
    }
    \end{subfigure}
    \begin{subfigure}[ ]{
        \centering
        \includegraphics[width=0.24\textwidth]{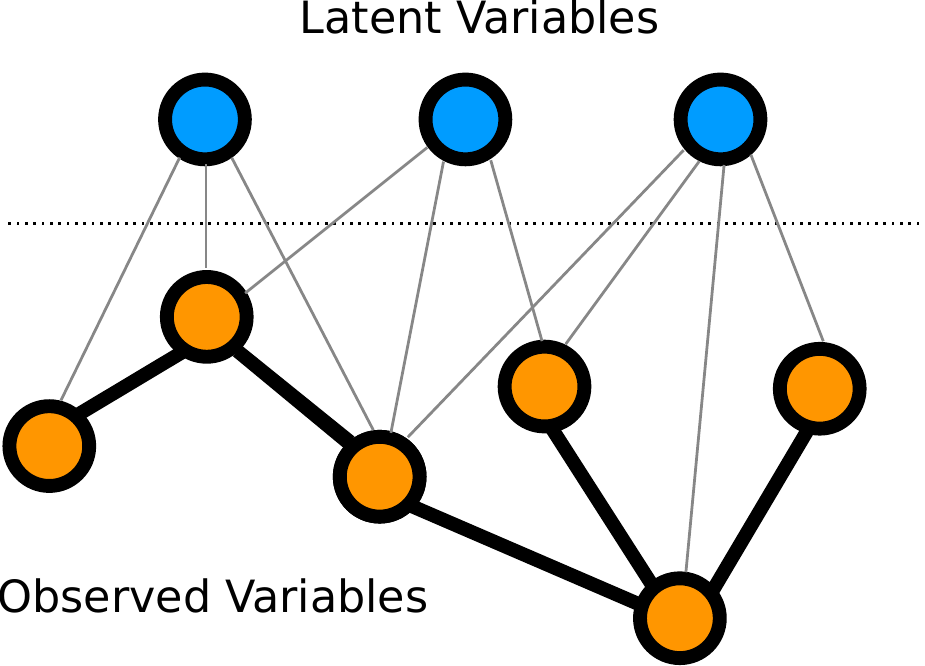}\label{subfig:toy:lvggm}    
    }
    \end{subfigure}
    \begin{subfigure}[ ]{
        \centering
        \includegraphics[width=0.20\textwidth]{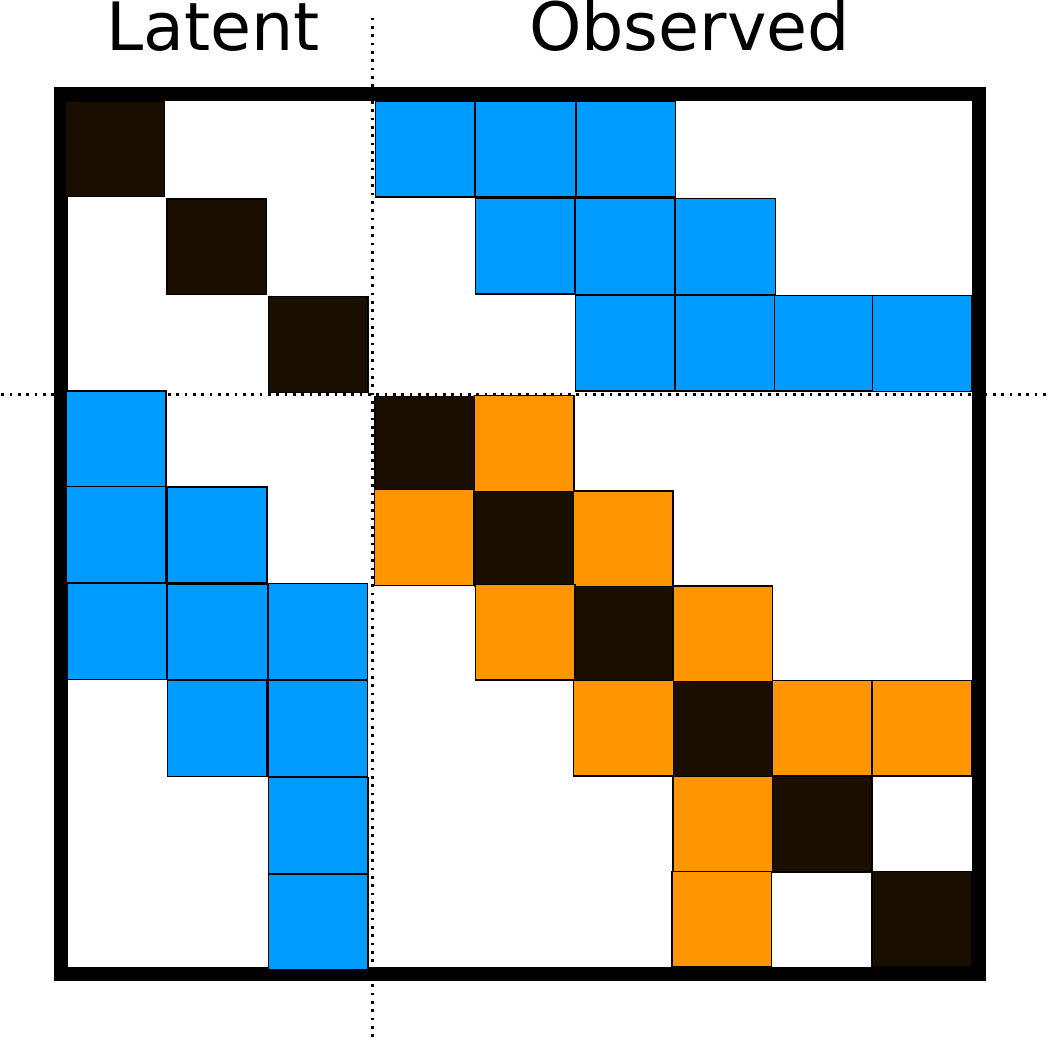}    \label{subfig:toy:lvggmMatrix}
    }
    \end{subfigure}
\caption{Illustrations of a sparse Gaussian graphical model (GGM) (left) and a latent variable Gaussian graphical model (LVGGM) (right).  (A) Example of a sparse GGM with only observed variables, (B) Sparsity pattern of example sparse GGM's precision matrix, (C) Example of a LVGGM with both observed and latent variables, (D) Sparsity pattern of example LVGGM's precision matrix.}
\end{figure*}

%------------------

In this section, we review Gaussian graphical models and formulate the problem of latent variable Gaussian graphical model estimation via a regularized maximum likelihood optimization. 

\subsection{Gaussian Graphical Models}
Consider a $p$-dimensional random vector $\mbf{x}$ associated with an undirected graph $\mcal{G} = (V_G,E_G)$, where $V_G$ is a set of nodes corresponding to elements of $\mbf{x}$ and $E_G$ is a set of edges connecting nodes (including self-edges for each node). 
Then $\mbf{x}$ follows a graphical model distribution if it satisfies the Markov property with respect to $\mcal{G}$: for any pair of nonadjacent nodes in $\mcal{G}$, the corresponding pair of variables in $\mbf{x}$ are conditionally independent given the remaining variables, \emph{i.e.,} $x_i \perp x_j \ | \ \mbf{x}_{\bs i, j}$, for all $(i, j) \notin E_G$.

If $\mbf{x}$ follows a multivariate Gaussian distribution, the corresponding graphical model is called a Gaussian graphical model (GGM).  We assume without loss of generality that $\mbf{x}$ has zero mean.  
The Markov property in GGM is manifested in the sparsity pattern of the inverse covariance matrix $\mbf{J}$:
\ALGNN{
\mbf{J}_{i,j} = 0 \text{ for all }   i \neq j, (i,j) \notin E. 
%\label{eq:markov}
}
An example of this property for sparse GGM is shown in Figure~\ref{subfig:toy:ggm} and~\ref{subfig:toy:ggmMatrix}.

The precision matrix parameterization arises in many statistical inference problems for Gaussian distributions, in areas such as belief propagation~\citep{malioutov2006walk}, 
linear prediction,  
portfolio selection in financial data~\citep{ledoit2003improved}, 
and anomaly detection~\citep{chen2011robust}. 
Estimation of the precision matrix in GGM is the first step in these inference problems.

\subsection{Latent Variable Gaussian Graphical Models}
\label{sec:LVGGMsetup}

Unfortunately, due to the presence of global factors that destroy sparsity, real-world observations often do not conform exactly to a sparse GGM~\citep{choi2010gaussian,choi2011learning}.  By introducing latent variables (denoted as a $r$-dimensional random vector $\mbf{x}_L$) to capture global factors, we can generalize the GGM.  
Specifically, we construct a model that is conditionally a GGM, {\em i.e.,} one that has a sparse precision matrix given knowledge of latent variables, ${\bf x}_L$. 

Defining the $p$ observed variables as $\mbf{x}_O$, we assume the joint distribution of the $(p+r)$-dimensional concatenated random vector $\mbf{x} = (\mbf{x}_{O}, \mbf{x}_{L})$ follows a Gaussian distribution with covariance matrix $\mbf{\Omega}$ and precision matrix $\mbf{J} = \mbf{\Omega}^{-1}$.   An example of this structure can be seen in Figure~\ref{subfig:toy:lvggm} and~\ref{subfig:toy:lvggmMatrix}. 
Marginalizing over the latent variables $\mbf{x}_L$, the distribution of the observed variables $\mbf{x}_O$ remains Gaussian with observed covariance matrix, $\mbf{\Sigma} = \mbf{\Omega}_{O, O}$.  The observed precision matrix $\Theta \in \mathbb{R}^{p \times p}$ satisfies:
\ALGNN{
\Theta & = \mbf{\Sigma}^{-1} = \underbrace{\mbf{J}_{O,O}}_{\mbf{S}} \underbrace{- \mbf{J}_{O,L} \mbf{J}_{L,L}^{-1} \mbf{J}_{L,O}}_{\mbf{L}},
\label{eq:theta}
} 
where we have defined $\mbf{S} := \mbf{J}_{O,O}$ and $\mbf{L} := - \mbf{J}_{O,L} \mbf{J}_{L,L}^{-1} \mbf{J}_{L,O}$. Thus, the marginal precision matrix can be written as $\Theta = \mbf{S} + \mbf{L}$, the sum of a sparse and a low-rank matrix.
Similar to standard GGM, we parameterize the marginal distribution through the precision matrix.   We refer to this model as the latent variable GGM, or LVGGM.

The LVGGM is a hierarchical model that generalizes the (sparse) GGM.  Note that $\mbf{S}^{-1} = \mbf{J}_{O,O}^{-1} = \mbf{\Omega}_{O, O} - \mbf{\Omega}_{O, L} \mbf{\Omega}_{L,L}^{-1} \mbf{\Omega}_{L, O}$ is the covariance matrix of the conditional distribution of the observed variables.  The matrix is not generally sparse, even though $\mbf{S}$ is assumed to be sparse.  We will also assume that the number of latent variables is much smaller than the number of observed variables, \emph{i.e.,} $r \ll p$.  We place no sparsity restrictions on the dependencies between the observed and latent variables -- the submatrices $\mbf{J}_{O,L}$ and $\mbf{J}_{L,O}$ could be dense. As a result, the $p \times p$ matrix $\mbf{L} = -\mbf{J}_{O,L} \mbf{J}_{L,L}^{-1} \mbf{J}_{L,O}$ is low-rank and potentially dense. The sparse plus low-rank structure of the marginal precision matrix $\Theta$ is the key property of the precision matrix that will be exploited for model estimation. 

The structural assumptions on the precision matrix of the LVGGM can be further motivated and validated on real-world recommender system data and stock return data.  Due to the space limits, we defer these two motivating examples to Section~\ref{sec:motivatingExp} in the Appendix.

\subsection{Effective Rank of Covariance Matrix}
\label{sec:effectiveRank}

We introduce the \emph{effective rank} of a matrix, which will be useful to derived high-dimensional error bounds.  The effective rank of a matrix $\mbf{\Sigma}$ is defined as~\citep{vershynin2010introduction}:
\ALGNN{
\reff(\mbf{\Sigma}) := {\text{tr}(\mbf{\Sigma})}/{\| \mbf{\Sigma} \|_2}.
}
%where $\| \cdot \|_2$ is the matrix spectral norm. 
The effective rank can be considered a measure of the concentration level of the spectrum of $\mbf{\Sigma}$. As we will show in Section~\ref{sec:effrankexp}, in many situations the effective rank of the covariance matrix corresponding to a LVGGM is much smaller than $p$. Under this condition, our theoretical results in the sequel provide a tight Frobenius norm estimation error bound, which is significantly improved upon the error bound derived without the effective rank assumption.

\subsection{Regularized ML Estimation of LVGGM}
\label{sec:regML}

Available are $n$ samples $x_1, x_2, \ldots, x_n$ from a LVGGM model $\mbf{x}_O$, concatenated into a data matrix $\mbf{X} \in \mathbb{R}^{p \times n}$.  The negative log-likelihood function is
\ALGNN{
\mcal{L}(\Theta; \mbf{X}) = \langle \Sighat, \Theta \rangle - \log \det (\Theta),
\label{eq:llh}
}
where $\Sighat := \frac{1}{n} \mbf{X}^T \mbf{X}$ is the sample covariance matrix. The regularized ML estimate minimizes the objective function $\mcal{L}(\Theta; \mbf{X}) + \la \mcal{R}(\Theta)$, where the regularization parameter $\lambda > 0$, and the regularization function $\mcal{R} ( \Theta )$ is designed to enforce the sparse plus low-rank structure on $\Theta$.

Similar to~\citet{chandrasekaran2012latent}, we consider the following regularized ML estimation problem:
\ALGNN{
\begin{split}
\min_{\mbf{S}, \mbf{L}} \ & \ \mathcal{L}(\mbf{S} + \mbf{L}; \mbf{X}) + \lambda \| \mbf{S} \|_1 + \mu \| \mbf{L} \|_* \\%\mathcal{R}(\mbf{S}, \mbf{L}), \\
{\text{s.t.}} \ & \ - \mbf{L} \succeq \mbf{0}, \ \mbf{S} + \mbf{L} \succeq \mbf{0},
\end{split} \label{eq:MestLVGGM}
}
where the corresponding regularization function is the sum of two regularizers:
$\mathcal{R}(\Theta) = \| \mbf{S} \|_1 + \frac{\mu}{\lambda} \| \mbf{L} \|_*$,
each of which has been shown to promote sparse (low-rank) structure in $\mbf{S}$ ($\mbf{L}$, respectively)~\citep{negahban2012unified}. Constants $\la, \mu >0$ are regularization parameters corresponding to the two functions, respectively.  The LVGGM estimator is defined as a solution to the above convex optimization problem~\eqref{eq:MestLVGGM}. Efficient convex solver, such as~\citet{ma2012alternating}, can be used to solve.

\section{Error Bounds on ML LVGGM Estimation}
\label{sec:theory}

We analyze the regularized ML estimation problem~\eqref{eq:MestLVGGM} and provide Frobenius norm error bounds for estimating the precision matrix in high-dimensional setting.  We adopt the decomposable regularization framework of~\cite{negahban2012unified, agarwal2012noisy, yang2013dirty} to derive these bounds. In contrast to this prior work, here we focus on multiple decomposable regularizers interacting with the non-quadratic log-likelihood loss function encountered in the LVGGM. 
Two important ingredients in the derivations are the restricted strong convexity of the loss function, and an incoherence condition between the two structured subspaces containing the sparse and low-rank components ($\mbf{S}$ and $\mbf{L}$).  
We show that under assumptions on the Fisher information these two conditions are verified. 

In the following subsections, first we define some necessary notation, then we introduce the assumptions and place them in the context of prior literature, and finally we state the main results in Theorem~\ref{thm:main} and Theorem~\ref{cor:errorbound:loweffrank}.

\subsection{Decomposable Regularizers and Subspace Notation}

In this subsection we introduce the notion of decomposable regularizers and the corresponding subspace pairs. We refer the reader to~\citet{negahban2012unified} for more details. 

Consider a pair of subspaces $(\mcal{M}, \overline{\mcal{M}}^{\perp})$, where $\mcal{M} \subset \overline{\mcal{M}} \subset \mathbb{R}^{p \times p}$. $\mcal{R}(\cdot)$ is called a decomposable regularization function with respect to the subspace pair if, for any $u \in \mcal{M}, v \in \overline{\mcal{M}}^{\perp}$, we have $\mcal{R}(u + v) = \mcal{R}(u) + \mcal{R}(v)$.
 
For the sparse and low-rank matrix-valued parameters, the following two subspace pairs and their corresponding decomposable regularizers are considered:

\begin{itemize}[leftmargin=10pt, itemsep=0pt, topsep=0pt]
\item \emph{Sparse matrices.} Let ${E} \subseteq \{1, \ldots, p\} \times \{1, \ldots, p\}$ be a subset of index pairs (edges). Define $\mcal{M}({E}) = \overline{\mcal{M}}({E})$ as the subspace of all sparse matrices in $\mathbb{R}^{p \times p}$ that are supported in subsets of ${E}$, \emph{i.e.,} $\mcal{P}_{\mcal{M}({E})}(\mbf{A}) = \mbf{A}_{{E}}$.  A decomposable regularizer is the $\ell_1$ norm, since $\| \mbf{A} \|_1 = \| \mbf{A}_{{E}} \|_1 + \| \mbf{A}_{{E}^C} \|_1$. 

\item \emph{Low-rank PSD matrices.} 
Consider a class of low-rank and positive semi-definite matrices $\mcal{A} \subset \mathbb{S}^{p \times p}_+$ which have rank $r \le p$. For any given matrix $\mbf{A} \in \mcal{A}$, let $\text{col}(\mbf{A})$ denote its column space.  
Let $U \subset \mathbb{R}^n$ be a $r$-dimensional subspace and define the subspace $\mcal{M}(U)$ and the perturbation subspace $\overline{\mcal{M}}^{\perp}(U)$ as
\ALGN{
\mcal{M}(U)  := & \{ \mbf{A} \in \mathbb{R}^{n \times p} \ | \ \text{col}(\mbf{A}) \subseteq U\}, \\
\overline{\mcal{M}}^{\perp}(U) := & \{ \mbf{A} \in \mathbb{R}^{n \times p} \ | \ \text{col}(\mbf{A}) \subseteq U^{\perp}\}.
}
Then the nuclear norm $\mcal{R}_L(\cdot) = \| \cdot \|_*$ is a decomposable regularization function with respect to the subspace pair $(\mcal{M}(U), \overline{\mcal{M}}^{\perp}(U))$. 
\end{itemize}

For the true model parameter $\Theta^*$, we define its associated \emph{structural error set} with respect to a subspace $\mcal{M}$ as~\citep{negahban2012unified}: 
%$\mathbb{C}(\mcal{M}, \overline{\mcal{M}}^{\perp};\Theta^*) :=$ as:
\ALGNN{
%\mathbb{C}(\mcal{M}, \overline{\mcal{M}}^{\perp}; \Theta^*) &:=& \\
%& & \hspace{-3 cm}
%\left\{ \Delta \in \mathbb{R}^{n \times p} \ | \ \mcal{R}(\Delta_{\overline{\mcal{M}}^{\perp}}) \le 3 \mcal{R}(\Delta_{\overline{\mcal{M}}}) + 4 \mcal{R}(\Theta^*_{\overline{\mcal{M}}^{\perp}}) \right\} \notag.
\mathbb{C}(\mcal{M}, \overline{\mcal{M}}^{\perp}; \Theta^*) := \left\{ \Delta \in \mathbb{R}^{n \times p} \ | \ \mcal{R}(\Delta_{\overline{\mcal{M}}^{\perp}}) \le 3 \mcal{R}(\Delta_{\overline{\mcal{M}}}) + 4 \mcal{R}(\Theta^*_{\overline{\mcal{M}}^{\perp}}) \right\} \notag.
}
By construction, if the norm of the projection of the true parameter $\Theta^*$ into $\overline{\mcal{M}}^{\perp}$ is small, then elements $\Delta$ in this structural error set also have limited projection onto the perturbation subspace $\overline{\mcal{M}}^{\perp}$. 

Now let $\Thetastar$ be the true (marginal) precision matrix of the LVGGM, and let the sparse and low-rank components be $\Sstar$ and $\Lstar$, respectively. 
For the defined subspace pairs $(\mcal{M}(E), \overline{\mcal{M}}(E)^{\perp})$ and $(\mcal{M}(U), \overline{\mcal{M}}(U)^{\perp})$, we use $\mathbb{C}(E)$ and $\mathbb{C}(U)$ as the shorthand notations for the corresponding structural error sets centered at $\Sstar$ and $\Lstar$, \emph{i.e.,} $\mathbb{C}(\mcal{M}(E), \overline{\mcal{M}}(E)^{\perp}; \Sstar)$ and $\mathbb{C}(\mcal{M}(U), \overline{\mcal{M}}(U)^{\perp}; \Lstar)$, respectively. Later, we will consider the perturbation of $\Thetastar$ along restricted directions in these two sets.

\subsection{Assumptions on Fisher Information}
\label{sec:assumpFisher}

We characterize the interaction between the elements in the two subspaces through their inner products using the Hessian of the loss function, also known as the \emph{Fisher information} of the distribution.  Denoting the Fisher information matrix of a Gaussian distribution as $\mcal{F}^*$ (evaluated at $\Thetastar$), we find that $\mcal{F}^* = {\Thetastar}^{-1} \otimes {\Thetastar}^{-1}$, where $\otimes$ is the Kronecker product. 
%It can be derived that~\cite{ravikumar2011high} 
We define the \emph{Fisher inner product} between two matrices $\Delta_{A}$ and $\Delta_B$ as
\ALGNN{
\langle \Delta_{A}, \Delta_B \rangle_{\mcal{F}^*} & := \text{vec}(\Delta_A)^T \mcal{F}^* \text{vec}(\Delta_B) \\
& = \text{Tr}({\Thetastar}^{-1} \Delta_A {\Thetastar}^{-1} \Delta_B),
}
where $\text{vec}(\cdot)$ denotes the vectorization of a matrix.

Similar to prior work of~\cite{kakade2009learning}, we define the induced \emph{Fisher norm} of a matrix $\Delta$ as
\ALGNN{
\| \Delta \|_{\mcal{F}^*}^2 & := \text{vec}(\Delta)^T \mcal{F}^* \text{vec}(\Delta) \label{eq:fishernorm} \\
& = \text{Tr}({\Thetastar}^{-1} \Delta {\Thetastar}^{-1} \Delta).
}

The first assumption we make is the following {\em Restricted Fisher Eigenvalue} (RFE) condition on the true precision model with respect to the sparse and low-rank structural error sets. 
\begin{assumption}[\textbf{Restricted Fisher Eigenvalue}]
\label{assump:rfe}
There exists some constant $\kappa_{\min}^* > 0$, such that for all $\Delta\in \mathbb{C}(E) \cup \mathbb{C}(U)$, the following holds:
\ALGNN{
\| \Delta \|_{\mcal{F}^*}^2 \ge \kappa_{\min}^* \| \Delta \|_F^2 \ .
\label{eq:RFE}
}
\end{assumption}

This RFE condition generalizes the restricted eigenvalue (RE) condition for sparsity-promoting linear regression problems~\cite{bickel2009simultaneous}.  It assumes that the minimum eigenvalue of the Fisher information is bounded away from zero along the directions $\mathbb{C}(E)$ and $\mathbb{C}(U)$. Due to the identity~\eqref{eq:fishernorm} and properties of the Kronecker product, a trivial lower bound for $\kappamin$ is $\la_{\min}^2(\Thetastar)$, where $\la_{\min}(\cdot)$ denotes the minimum eigenvalue. In the high-dimensional setting, the RFE parameter $\kappamin$, which is defined only with respect to the above restricted set of directions, can be substantially larger than $\la_{\min}^2(\Thetastar)$.  As a result, the derived error bounds, which depend on $\kappamin$, are generally tighter than the bounds depending on $\la_{\min}^2(\Thetastar)$ (cf.~Theorem~\ref{thm:main}).

Due to the sparse plus low-rank superpositioned structure, we impose a type of incoherence between the two structural error sets to ensure consistent estimation of the combined model. The incoherence condition will limit the interaction between elements from the two sets. 
For our problem, such interaction occurs through their inner products with the Fisher information, which motivates the following \emph{Structural Fisher Incoherence} (SFI) assumption (which generalizes the \emph{C-Linear} assumption proposed in~\citet{yang2013dirty}).

Let $\mcal{P}_{E} := \mcal{P}_{\overline{\mcal{M}}(E)}$ denote the projection operator corresponding to the subspace $\overline{\mcal{M}}(E)$. Similarly define $\mcal{P}_U := \mcal{P}_{\overline{\mcal{M}}(U)}$, $\mcal{P}_{E^{\perp}} := \mcal{P}_{\overline{\mcal{M}}(E)^{\perp}}$, and $\mcal{P}_{U^{\perp}} := \mcal{P}_{\overline{\mcal{M}}(U)^{\perp}}$. We assume the following condition on the Fisher information. 

\begin{assumption}[\textbf{Structural Fisher Incoherence}]
\label{assump:sfi}
Given a constant $M>6$, a set of regularization parameters $(\la, \mu)$, and the subspace pairs $(\mcal{M}(E), \overline{\mcal{M}}(E)^{\perp})$ and $(\mcal{M}(U), \overline{\mcal{M}}(U)^{\perp})$ as defined above, let $\Lambda = 2 + 3 \max \left\{ \frac{\la \sqrt{s}}{\mu \sqrt{r}},  \frac{\mu \sqrt{r}}{\la \sqrt{s}} \right\}$, where $s = |E|$ and $r = \text{rank}(U)$. 
%Suppose that for all $\Delta_S \in \mathbb{C}(E)$ and $\Delta_L \in \mathbb{C}(U,V)$ such that $\max \{ \| \Delta_S \|_{\mcal{F}^*}, \| \Delta_L \|_{\mcal{F}^*} \} \le \frac{1}{8 M^2}$ for some constant $M \ge 4$, 
%Suppose that the Fisher information $\mcal{F}^*$ satisfies:
Then the Fisher information $\mcal{F}^*$ satisfies:
\ALGN{
%\begin{split}
%& \max \left\{ \overline{\sigma} \left( \mcal{P}_{E} \mcal{F}^* \mcal{P}_{U} \right), \overline{\sigma} \left( \mcal{P}_{E^{\perp}} \mcal{F}^* \mcal{P}_{U} \right), \right. \\
%& \left. \overline{\sigma} \left( \mcal{P}_{E} \mcal{F}^* \mcal{P}_{U^{\perp}} \right), \overline{\sigma} \left( \mcal{P}_{E^{\perp}} \mcal{F}^* \mcal{P}_{U^{\perp}} \right) \right\} \le \frac{\kappa_{\min}^*}{c_1 \Lambda^2},
%\end{split}
& \max \left\{ \overline{\sigma} \left( \mcal{P}_{E} \mcal{F}^* \mcal{P}_{U} \right), \overline{\sigma} \left( \mcal{P}_{E^{\perp}} \mcal{F}^* \mcal{P}_{U} \right), \overline{\sigma} \left( \mcal{P}_{E} \mcal{F}^* \mcal{P}_{U^{\perp}} \right), \overline{\sigma} \left( \mcal{P}_{E^{\perp}} \mcal{F}^* \mcal{P}_{U^{\perp}} \right) \right\} \le \frac{\kappa_{\min}^*}{c_1 \Lambda^2},
} 
where $\overline{\sigma}(\cdot)$ denotes the maximum singular value, and constant $c_1$ is defined as $c_1 = \frac{16M}{M-6}$.
\end{assumption}

%General SFI comments
The constant $M$ is related to a ``burn-in'' period after which the likelihood loss function has desirable properties in a small neighborhood of the true parameter.  
In particular, when $M = 7$, the constant $c_1 = 112$ suffices for our theory to hold. See the main theorem and its proof for more discussion on this quantity.

It is interesting to compare our SFI assumption to other similar assumptions in the literature of GGM estimation. In \citet{ravikumar2011high}, a form of irrepresentability condition is assumed, which limits the induced $\ell_1$ norm of a matrix that is similar to the projected Fisher information onto the sparse matrix subspace pair. In \citet{chandrasekaran2012latent}, the notion of irrepresentability is extended to two subspace pairs (\emph{i.e.,} sparse and low-rank), but detailed behaviors of the projected Fisher information are controlled (see the main assumption on page 1949 of~\citet{chandrasekaran2012latent}). For model selection consistency, a more general form of irrepresentability has been shown to be necessary for model selection consistency, see~\citet{lee2013model} for a recent discussion. In contrast to the above line of work, the SFI assumption we make only controls the maximum singular values of the projected Fisher information. This can be explained as we are interested in bounding a weaker quantity, the Frobenius norm of the parameter estimation error, instead of establishing the stronger model selection consistency of~\citet{ravikumar2011high} or the algebraic consistency as in~\citet{chandrasekaran2012latent}.

\subsection{Error Bounds for LVGGM Estimation}

We have the following bound on the parameter error of the estimated precision matrix of LVVGGM, $\Thetahat = \Shat + \Lhat$, obtained by solving the regularized ML problem~\eqref{eq:MestLVGGM}.

\begin{thm}
\label{thm:main}
Suppose Assumption~\ref{assump:rfe} and~\ref{assump:sfi} hold for the true marginal precision matrix $\Thetastar$, and the regularization parameters are chosen such that
\ALGNN{
\la \ge 2 \| \Sigstar - \Sighat \|_{\infty} \ \text{ and } \ \mu \ge 2 \| \Sigstar - \Sighat \|_{2}.    \label{eq:regparams}
}
Given a constant $M > 6$, if an optimal solution pair $(\Shat, \Lhat)$ to the convex program~\eqref{eq:MestLVGGM} satisfies 
\ALGNN{
\max \{ \| \Shat - \Sstar \|_{\mcal{F}^*}, \| \Lhat - \Lstar \|_{\mcal{F}^*} \} \le \frac{1}{6 M^2},    \label{eq:burnin}
}
then we have the following error bound for the estimated precision matrix $\Thetahat = \Shat + \Lhat$:
\ALGNN{
\label{eq:errorbound}
\| \Thetahat - \Thetastar \|_{F} \le \frac{6}{\kappaL} \max \left\{ \la \sqrt{s}, \mu \sqrt{r} \right\} + \sqrt{\frac{8 r_{\perp}^*}{\kappaL}},
}
where $s = |E|$, $r = \textrm{rank}(U)$, and 
\ALGNN{
\kappaL := & \ \frac{M-2}{2(M-1)} \kappa_{\min}^*, \\
r_{\perp}^* := & \ \la \sum_{(j,k) \notin E} | \Sstar_{jk} | + \mu \sum_{j = r+1}^p \sigma_j ( \Lstar).
}
\end{thm}

\begin{proof}[Proof sketch]
The proof is inspired by~\citet{yang2013dirty}, in which a parameter estimation error bound is proven for estimating a class of superposition-structured parameters, such as sparse plus low-rank, through M-estimation with decomposable regularizers. 
Critical to specializing this framework to our LVGGM estimation problem is to verify two conditions on the log-likelihood loss function~\eqref{eq:llh}: the restricted strong convexity (RSC) and structural incoherence (SI). The RSC condition (which originally proposed in~\citet{negahban2012unified}) specifies the loss function to be sufficiently curved ({\em i.e.}~lower bounded by a quadratic function) along a restricted set of directions (defined by $\mathbb{C}(E)$ and $\mathbb{C}(U)$). On the other hand, the SI condition effectively limits certain interaction between elements from the above two structural error sets. In~\citet{yang2013dirty}, under certain \emph{C-linear} assumptions, the RSC and SI conditions are verified for several problems with quadratic loss functions. For the LVGGM estimation problem, however, the technical difficulty lies in the non-quadratic log-likelihood loss~\eqref{eq:llh}, for which the previously established RSC and SI conditions do not hold. 

To deal with this difficulty, we leverage the \emph{almost strong convexity} properties~\citep{kakade2009learning} to characterize the convergence behavior of the sum of higher-order terms in the Taylor series of the log-likelihood loss function. We show that in the regime specified by condition~\eqref{eq:burnin}, the loss function can be well-approximated by the sum of a quadratic function and a residual term. %(along the restricted set of directions $\mathbb{C}(E)$ and $\mathbb{C}(U)$).  
Under this condition, the RFE assumption (Assumption~\ref{assump:rfe}) guarantees the RSC condition (cf.~Lemma~\ref{lem:rsc}), and the SFI assumption (Assumption~\ref{assump:sfi}) leads to SI condition to hold (cf.~Lemma~\ref{lem:fisher:inner}). Theorem~\ref{thm:main} can then be proven by the general theorem in~\citet{yang2013dirty}. A detailed proof of Theorem~\ref{thm:main} can be found in Appendix~\ref{sec:proof:main}. 
\end{proof}

We make the following remarks:
%\begin{description}[leftmargin=0pt, itemsep=0pt, topsep=0pt]
\begin{itemize}%[leftmargin=10pt, itemsep=0pt, topsep=0pt]
\item The error bound~\eqref{eq:errorbound} is a family of upper bounds defined by different sets of subspace pairs $(\mcal{M}(E), \overline{\mcal{M}}(E)^{\perp})$ and $(\mcal{M}(U), \overline{\mcal{M}}(U)^{\perp})$.  The tightest bound can be achieved by appropriately choosing $E$ and $U$. The first additive term in~\eqref{eq:errorbound} captures effect of the estimation error, while the second term captures the approximation error. In many cases it is reasonable to assume the approximation error is zero, then the error bound reduces to the first additive term. 

\item We note that similar derivations also apply to $\ell_1$-regularized estimation of sparse GGM. For the sparse GGM, only Assumption~\ref{assump:rfe} is required, and the derivations largely simplify. The final error bound also contains estimation and approximation errors, depending only on the sparse matrix subspace pair. However, when the true precision matrix $\Thetastar$ cannot be well-approximated as a sparse matrix (such as the LVGGM case), the approximation error would be much worse, leading to an inefficient learning rate.   

\item We finally remark that the SFI assumption can be relaxed to an even milder incoherence condition, $\| \mbf{L} \|_\infty \le \alpha$, as considered in~\citet{agarwal2012noisy}. Following similar derivations as in the proof of Theorem~\ref{thm:main}, the corresponding error bound can be obtained. However, as a result of this incoherence assumption, the error bound would contain an additional incoherence term which does not vanish to zero even with infinite samples. This disadvantage is overcome under the structural incoherence condition.
%~\citep{yang2013dirty}, which is satisfied under our SFI assumption. 

\end{itemize}
%\end{description}

The statement of Theorem~\ref{thm:main} is deterministic in nature and applies to any optimum of the convex program. However, the condition on the regularization parameters~\eqref{eq:regparams} and the error bound depend on the sampled data (in particular the sample covariance matrix $\widehat{\Sigma}$), which is random. 
Therefore the key to specifying the regularization parameters, and hence obtaining error bounds independent of data, is to derive tight deviation bounds of the sample covariance matrix in terms of the $\ell_{\infty}$ and $\ell_2$ norms, such that condition~\eqref{eq:regparams} holds with high probability. These bounds can be obtained by using concentration inequalities for Gaussian distributions, which leads to the following corollary. % of Theorem~\ref{thm:main}.

\begin{cor}    \label{cor:main}
Let the same assumptions in Theorem~\ref{thm:main} hold.  Given constants $C_1 > 1$ and $C_2 \ge 1$, assume that the number of samples $n$ satisfies 
%\ALGNN{
$n \ge \max \left\{ 4 C_1^2 \log p, \ C_2^2 p \right\}$, 
%\label{eq:samplereq:cor}
%}
and that the regularization parameters satisfy 
\ALGNN{
 \lambda = 160 C_1 \overline{\sigma}^* \sqrt{\frac{\log p}{n}} \ \text{ and } \ \mu = 16 C_2 \rho^* \sqrt{\frac{p}{n}}, 
 \label{eq:regcond:cor1} 
}
where $\overline{\sigma}^* = \max_{i} \Sigstar_{i,i}$ and $\rho^* = \| \Sigstar \|_{2}$. Then with probability at least $1 - 4 p^{- 2(C_1 - 1)} - 2 \exp ( - \frac{C_2^2 p}{2})$, we have
%\ALGNN{
%\label{eq:errorbound:stochastic}
%\| \Thetahat - \Thetastar \|_{F} \le \frac{6}{\kappaL} \max \left\{ 2 C_{1} \sqrt{\frac{2 s \log p}{n}}, 2 C_{2} \sqrt{\frac{r p}{n}} \right\}.
%% + 2 \sqrt{\frac{2 r_{\perp}}{\kappaL}},
%}
\ALGNN{
\label{eq:errorbound:stochastic}
\| \Thetahat - \Thetastar \|_{F} \le c_1 \sqrt{\frac{s \log p}{n}} + c_2 \sqrt{\frac{r p}{n}},
% + 2 \sqrt{\frac{2 r_{\perp}}{\kappaL}},
}
where $c_1 = \frac{960}{\kappaL} \overline{\sigma}^*$ and $c_2 = \frac{96}{\kappaL} \rho^*$. 
\end{cor}

\emph{Remark:} The estimation error~\eqref{eq:errorbound:stochastic} consists of two terms corresponding to the sparse and low-rank components, respectively. Note its resemblance to the error bounds of robust PCA (\emph{e.g.,}~\citet{agarwal2012noisy,yang2013dirty}) and the derived bound in~\citet{chandrasekaran2012latent}. In particular, the first term in~\eqref{eq:errorbound:stochastic} was on the same order as the estimation error of a sparse GGM~\citep{ravikumar2011high}. However, due to the presence of latent variables, both the sample requirement (\emph{i.e.,} $n \gtrsim p$) and the combined error bound are worse than those for learning the sparse conditional GGM. %suggesting the additional difficulty in learning LVGGMs. 

Next we consider a scenario under which this additional disadvantage  is largely removed. Assume that the true marginal covariance matrix $\Sigstar$ has an effective rank $\reff := \reff(\Sigstar)$ (recall $\reff(\Sigstar) := \text{tr}(\Sigstar) / \| \Sigstar \|_2$ ) that is much smaller than $p$. Then, by using recent advances on the asymptotic behavior of the sample covariance matrix~\citep{lounici2012high}, we can obtain a much tighter bound which only depends on $p$ logarithmically, as stated in the following theorem.

\begin{thm}    \label{cor:errorbound:loweffrank}
Let the same assumptions in Theorem~\ref{thm:main} hold. Given a constant $C_1 > 1$, assume that the number of observations $n$ satisfies 
%\ALGNN{
$n \ge \max \left\{ 4 C_1 \log p, \ C_3 \reff \log^2 (2p) \right\}$, 
%\label{eq:samplereq:cor}
%}
and the regularization parameters satisfy
\ALGNN{
\lambda = 160 C_1 \overline{\sigma}^* \sqrt{\frac{\log p}{n}} \ \text{ and } \ \mu =  C_4 \rho^* \sqrt{\frac{\reff \log p}{n}},  
\label{eq:regcond:cor2}
}
where $\overline{\sigma}^* = \max_{i} \Sigstar_{i,i}$, $\rho^* = \| \Sigstar \|_{2}$, and $C_3, C_4 > 0$ are sufficiently large constants. Then with probability at least $1 - 2 p^{-2 (C_1 - 1)} - (2p)^{-1}$, we have
%\ALGNN{
%\label{eq:errorbound:stochastic}
%\| \Thetahat - \Thetastar \|_{F} \le \frac{6}{\kappaL} \max \left\{ 2 C_{1} \sqrt{\frac{2 s \log p}{n}}, 2 C_{2} \sqrt{\frac{r p}{n}} \right\}.
%% + 2 \sqrt{\frac{2 r_{\perp}}{\kappaL}},
%}
\ALGNN{
\label{eq:errorbound:stochasticnew}
\| \Thetahat - \Thetastar \|_{F} \le \tilde{c}_1 \sqrt{\frac{s \log p}{n}} + \tilde{c}_2  \sqrt{\frac{\reff \cdot r \log(2p)}{n}},
% + 2 \sqrt{\frac{2 r_{\perp}}{\kappaL}},
}
where $\tilde{c}_1 = \frac{960}{\kappaL} \overline{\sigma}^*$, $\tilde{c}_2 = \frac{8 C_4 }{3 \kappaL} \rho^*$. 
\end{thm}

\begin{proof}[Proof sketch] 
Same as Corollary~\ref{cor:main}, we need to verify that the choices of regularization parameters~\eqref{eq:regcond:cor2} satisfy the condition~\eqref{eq:regparams} with high probability. Since the choice of $\la$ has been verified in Corollary~\ref{cor:main}, it only remains to verify the condition on $\mu$. 
To this end, we make use of the following sharp bound on the spectral norm deviation of the sample covariance matrix:
%\footnote{We have modified the original statement by removing the assumption of missing observations.}:
\begin{lem}[\citet{lounici2012high}]
Let $\Sighat$ be a sample covariance matrix constructed from $n$ \emph{i.i.d.} samples from a $p$-dimensional Gaussian distribution $\mcal{N}(0, \Sigstar)$. Then with probability at least $1 - (2p)^{-1}$,
\ALGN{
& \| \Sighat - \Sigstar \|_2 \le C \| \Sigstar \|_2 \max \left\{\sqrt{\frac{2 \reff \log (2p)}{n}}, \frac{2 \reff \log (2p)(3/8 + \log (2pn)}{n} \right\},
}
where $C > 0$ is an absolute constant.
\end{lem}
Then as commented in~\citet{lounici2012high} (Prop.~3), when the sample size $n$ is sufficiently large such that $n \ge C_3 \reff \log^2 \max\{ 2p, n\}$, where $C_3 > 0$ is a large constant, the choice of regularization parameter $\mu$ as in~\eqref{eq:regcond:cor2} suffices for the condition~\eqref{eq:regparams} to hold with high probability.  \end{proof}

Notice that when $\reff \ll p$, the error bound~\eqref{eq:errorbound:stochasticnew} is significantly tighter than the bound~\eqref{eq:errorbound:stochastic}. Also the sample requirement $n \gtrsim \reff \log(p)$ is much milder. This result implies the efficiency of LVGGM learning when the true covariance model has a low effective rank.

%\section{Experiments}
%\label{sec:exp}
\section{Experiments}
\label{sec:exp}

We use a set of simulations on synthetic data to verify our reduced effective rank assumption on the covariance matrix of LVGGM, and the derived error bounds in Theorem~\ref{cor:errorbound:loweffrank}. 

\subsection{Effective Rank of Covariance of LVGGM}
\label{sec:effrankexp}

To better understand the effective rank of the covariance matrix of LVGGM, it is convenient to consider a hierarchical generating process for the observed variables: 
%\ALGNN{
${x}_O \sim \mbf{A} {x}_L + z$,
%\label{eq:genmodel}
%}
where ${x}_L \sim \mcal{N}(\mbf{0}, \mbf{\Omega}_{L,L})$ are the latent variables, $\mbf{A} := \mbf{J}_{O,O}^{-1} \mbf{J}_{O,L} \in \mathbb{R}^{p \times r}$, and ${z} \sim \mcal{N}(\mbf{0}, \mbf{S}^{-1})$ captures the conditional effects. 
The marginal covariance matrix of the observed variables can be represented as 
\ALGNN{
\mbf{\Sigma} = \underbrace{\mbf{A} \mbf{\Omega}_{L,L} \mbf{A}^{T}}_{\mbf{G}} + \mbf{S}^{-1},
\label{eq:covmodel}
}
where $\mbf{G}$ is a low-rank covariance matrix (global effects), and $\mbf{S}^{-1}$ is a non-sparse covariance matrix (conditionally local effects) whose inverse is sparse.  
%The covariance matrix $\mbf{\Sigma}$ has a low-rank plus sparse-inverse structure.  
While the low-rank global effects naturally result in a concentrated spectrum, the sparse-inverse local effects generally contribute to a diffuse spectrum. The effective rank, which is the sum of all eigenvalues divided by the magnitude of the largest one, depends on the relative energy ratio between $\mbf{G}$ and $\mbf{S}^{-1}$. 

Since an exact characterization of the effective rank in terms of $\mbf{A}$, $\mbf{\Omega}_{L,L}$, and $\mbf{S}$ tends to be difficult, we use Monte Carlo simulations to investigate synthetic LVGGM that conform to our assumptions. We generate LVGGM with independent latent variables (\emph{i.e.,} diagonal $\mbf{J}_{L,L}$), dense latent-observed submatrix $\mbf{J}_{L,O}$, and a sparse conditional GGM $\mbf{J}_{O,O}$ for observed variable with a random sparsity pattern (sparsity level $\approx 5\%$).  We fix the number of latent variables to be 10, and vary the number of observed variables $p = \{80, 120, 200, 500\}$. By scaling the magnitudes of the elements in the latent variable submatrix, we sweep through the relative energy ratio between the global and local factors, \emph{i.e.,} $\text{Tr}(\mbf{G}) / \text{Tr}(\mbf{S}^{-1})$ from 0.1 to 10.  After 550 realizations for each value of $p$, we plot the empirical effective ranks of observed covariance matrices in Figure~\ref{fig:effrank}. 

%%%%%%%%%%%%%%%%%%%%%%%%%%%%
%\begin{figure*}[t!]
%\centering
%\begin{minipage}[b]{.4\textwidth}
%\includegraphics[width=0.4\textwidth]{figures/EffRanks}
%\caption{Effective ranks of covariance matrices of LVGGM with various global/local energy ratios.}
%\label{fig:effrank}
%\end{minipage}\qquad
%\begin{minipage}[b]{.4\textwidth}
%\includegraphics[width=0.4\textwidth]{figures/errorPlot_collaps}
%\caption{Simulations for chain graphical models with latent variables. Plots of Frobenius norm error $\|\widehat{\Theta} - \Theta^*\|_F$ versus the rescaled sample size $n/(s \log(p) + \reff \cdot r \log(2p))$.}
%\label{fig:errorplot}
%\end{minipage}
%\end{figure*}
%%%%%%%%%%%%%%%%%%%%%%%%%%%%
%\vspace{-5pt}
\begin{figure}[t!]
\centering
\includegraphics[width=0.5\textwidth]{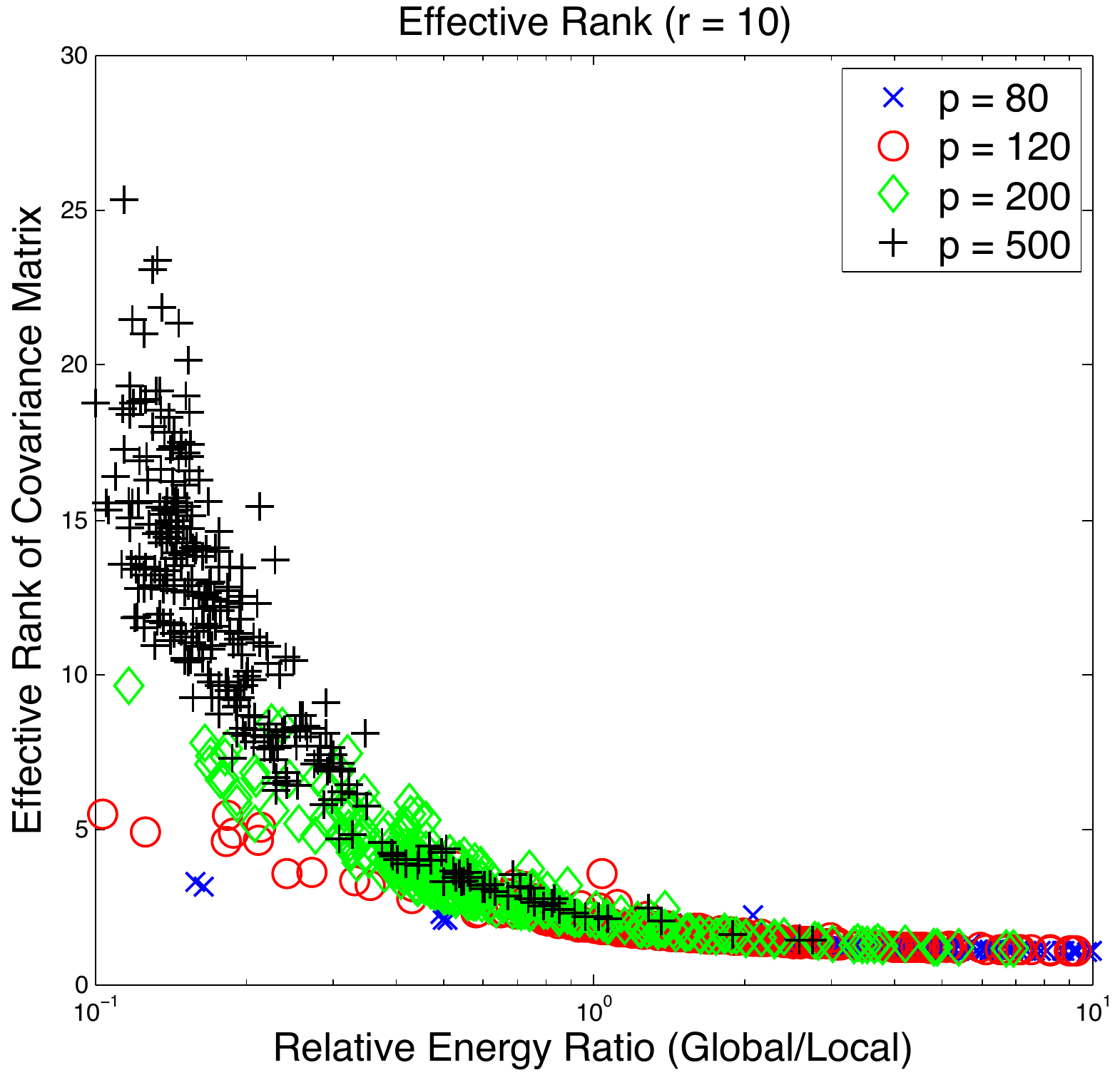}
%\vspace{-10pt}
\caption{Effective ranks of covariance matrices of LVGGM with various global/local energy ratios.}
%\vspace{-7pt}
\label{fig:effrank}
\end{figure}
%%%%%%%%%%%%%%%%%%%%%%%%%%%%

%%%%%%%%%%%%%%%%%%%%%%%%%%%%%%%%%%%%%%%%%%%%%%%%%%%%%%
\begin{figure}
\centering
%\begin{subfigure}[$\|\widehat{\Theta} - \Theta^*\|^2_F / \|\Theta^*\|^2_F \sim n$]
%    {\centering
%    \includegraphics[width=0.45\textwidth]{figures/errorPlot_separate}
%    \label{subfig:errorplot1}
%    }
%\end{subfigure}
%\begin{subfigure}[$\|\widehat{\Theta} - \Theta^*\|^2_F / \|\Theta^*\|^2_F \sim n/(s \log(p) + r \log(2p))$]
%    {\centering
    \includegraphics[width=0.6\textwidth]{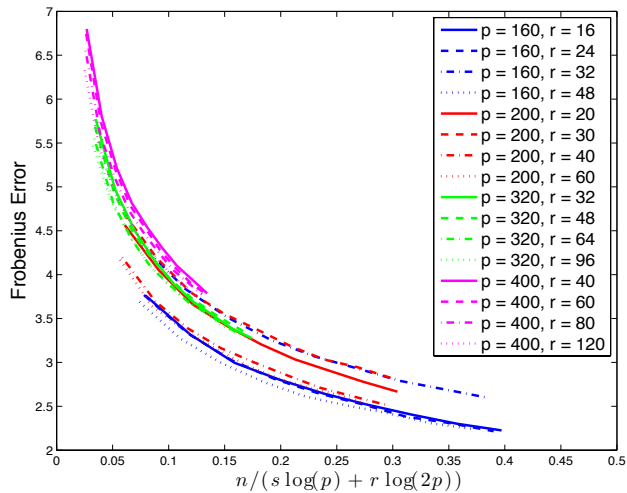}
%    \label{fig:errorplot2}
%    }
%\end{subfigure}
%\vspace{-10pt}
\caption{Simulations for chain graphical models with latent variables. Plots of Frobenius norm error $\|\widehat{\Theta} - \Theta^*\|_F$ versus the rescaled sample size $n/(s \log(p) + r \log(2p))$.}
%\vspace{-8pt}
\label{fig:errorplot}
\end{figure}
%%%%%%%%%%%%%%%%%%%%%%%%%%%%%%%%%%%%%%%%%%%%%%%%%%%%%%

As seen in the figure, when the global factor dominates (\emph{i.e.,} the ratio is large), the effective rank of the covariance matrix is very small, as expected.
%\footnote{We note that the effective rank is smaller than the number of latent variables in this experiment, likely due to the control of positive definiteness of $\mbf{J}$.}.  
On the other hand, when the local effects become stronger (\emph{e.g.,} when the number of observed variables $p$ increases) the effective rank increases, but at a very mild rate. In particular, when $p$ increases from 80 to 500, the maximum empirical effective rank in our simulation only increases from 4 to 26. For all of our simulated LVGGM, the empirical effective ranks are observed as at least an order of magnitude smaller than $p$. This mild growing rate of the effective rank (compared to $p$) will lead to our improved error bound in Theorem~\ref{cor:errorbound:loweffrank} to hold. %as opposed to that in Corollary~\ref{cor:main}.

\subsection{Frobenius Norm Error of LVGGM Estimation}

We simulate LVGGM data with number of observed variables $p = \{160, 200, 320, 400\}$ and number of latent variables in the set $r = \{0.1, 0.15, 0.2, 0.3\}p$. The sparse conditional GGM is a chain graph whose associated precision matrix is tridiagonal with off-diagonal elements $\mbf{S}_{i, i-1} = \mbf{S}_{i, i+1} =  0.4 \mbf{S}_{i,i}$ for $i = \{ 2, \ldots, p-1\}$. For each configuration of $p$ and $r$, we draw $n$ samples from the LVGGM, where $n$ ranges from 200 to 1000. Using these samples, the precision matrix $\widehat{\Theta}$ is learned by solving the regularized ML estimation problem~\eqref{eq:MestLVGGM}. As shown in Section~\ref{sec:effrankexp}, the effective rank of the covariance matrix grows mildly. Then Theorem~\ref{cor:errorbound:loweffrank} predicts that the Frobenius error of the estimated precision matrix of LVGGM should scale as $\| \widehat{\Theta} - \Theta^* \|_F \asymp \sqrt{(s \log (p) + r \log (2p))/n}$, when the regularization parameters are chosen such that $\la \asymp \overline{\sigma}^* \sqrt{\frac{\log (p)}{n}}$ and $\mu \asymp \rho^* \sqrt{\frac{\reff \log (p)}{n}}$. Guided by this theoretical result, we set the regularization parameters as $\la = C_a \overline{\sigma}^* \sqrt{\frac{\log (p)}{n}}$ and $\mu = C_b \rho^* \sqrt{\frac{\reff \log (p)}{n}}$, where constants $C_a$ and $C_b$ are cross-validated 
and then fixed for all test data sets with different configurations. We plot the Frobenius estimation errors against the rescaled sample size $n/(s \log(p) + r \log(2p))$ in Figure~\ref{fig:errorplot}. 
With a wide range of configurations, almost all the empirical error curves for models align and have the form of $f(t) \propto t^{-1/2}$ when the sample size is rescaled, as predicted by Theorem~\ref{cor:errorbound:loweffrank}.   
In practice when the true model is unknown, one could set the regularization parameters according to the sample versions of the quantities $\overline{\sigma}^*$ and $\rho^*$, as discussed in~\citet{lounici2012high}.

%\section{Conclusions}
%\label{sec:conclude}
\section{Conclusions}
\label{sec:conclude}

We consider a family of latent variable Gaussian graphical model whose precision matrix has a sparse plus low-rank structure. We derive parameter error bounds for regularized maximum likelihood estimation. Future work includes extending the framework to other distributions and the application to tasks such as prediction and detection.

%\vspace{-1em}
\section*{Acknowledgement}
{This work is in part supported by ARO grant W911NF-11-1-0391. 
The authors thank the anonymous reviewers for their valuable comments, along with Jason Lee and Yuekai Sun for helpful discussions.}

%\clearpage
\bibliography{lvggm_icml}
\bibliographystyle{icml2014}

%\onecolumn
\newpage
\appendix

\begin{center}
{\LARGE
\textbf{Appendix}
}
\end{center}

\section{Motivating Real-World Examples}
\label{sec:motivatingExp}

In this section, we use two real-world examples, movie rating and stock return price data sets, to motivate the LVGGM. For each data set, we manually choose three groups of variables where variables in one group are related. Effectively we have injected certain global effects, \emph{i.e.,} group effect, in the data. According to the decomposition of covariance matrix of a LVGGM (see Eq.~\eqref{eq:covmodel}), we examine whether these effects can be extracted using a low-rank component $\mbf{G}$ in the covariance matrix, and whether the remaining residual effects have a precision matrix $\mbf{S}$ that is sparser than its inverse $\mbf{S}^{-1}$.

We emphasize that for these two examples we are using  eigen-decomposition to decompose the covariance matrix into two components. However, this is not related to the regularized ML estimation algorithm proposed in Section~\ref{sec:regML}. The low-rank and sparse components that would be learned from the regularized ML problem are different to what we are showing here.

\textbf{Movielens data.} Using the \emph{Movielens}\footnote{\footnotesize{http://movielens.org}} movie rating data set, we choose the rating scores given by the most active 600 users and for the highest rated 20 movies from each of the following three genres: \emph{Horror}, \emph{Children's}, and \emph{Action}. This results in a $600 \times 60$ rating matrix with $56\%$ completeness. We consider the joint distribution of 60 movie rating variables as a LVGGM with three latent variables. Each user's rating vector is treated as an \emph{i.i.d.} sample from the LVGGM.  Since the true covariance matrix is unknown, we use the sample covariance matrix as a proxy (as $n \gg p$). Each covariance element is weighted by the actual number of observations to compensate for the missingness in the data.

To validate this intuition, we decompose the rating matrix into two matrices: a rank-3 matrix spanned by the top three leading singular vectors, and a residual matrix capturing the conditional effects. We denote the covariance matrix of the low-rank component as $\widetilde{\mbf{G}}$, and the sparse precision matrix of the residual component as $\widetilde{\mbf{S}}$. A heat map of the normalized $\widetilde{\mbf{G}}$ is shown in Figure~\ref{subfig:movie:G}, and the sparsity patterns of the normalized $\widetilde{\mbf{S}}$ and $\widetilde{\mbf{S}}^{-1}$ (\emph{i.e.,} the covariance of the residual) are shown in Figure~\ref{subfig:movie:S}, thresholded by $0.1$.  As expected, the low-rank $\widetilde{\mbf{G}}$ captures the structure of the global effects (\emph{i.e.,} genre), and the residual can be well-modeled by a sparse GGM -- as we observe that the precision matrix is much sparser than the covariance matrix. In addition, we find the effective rank of the covariance is equal to $7.4$, much smaller than the number of variables, $60$. 

\textbf{Stock return data.} Next, we validate the LVGGM assumptions on a monthly stock return data set\footnote{\footnotesize{http://people.csail.mit.edu/myungjin/latentTree.html}}, which consists of 216 samples of 24 stocks from three sectors: \emph{Technologies}, \emph{Industrials}, and \emph{Financials}. Similar to the \emph{Movielens} data, we reconstruct the low-rank component matrix $\widetilde{\mbf{G}}$ for the global effects (with rank = 4), and the sparse precision matrix $\widetilde{\mbf{S}}$ for the residual. The heat map of $\widetilde{\mbf{G}}$ and the sparsity patterns of $\widetilde{\mbf{S}}$ and $\widetilde{\mbf{S}}^{-1}$ are shown in Figure~\ref{subfig:stock:G} and~\ref{subfig:stock:S}, respectively. Again, the global structure (\emph{i.e.}, sector) is manifested in the low-rank matrix, and the conditional effects have a much sparser precision matrix than the covariance. We find the effective rank is equal to $2.9$, which again is much smaller than the total number of variables, $24$.

\begin{figure*}[h!]
%\centering
%    \begin{subfigure}[Sample Covariance Matrix $\Sigma$]{0.25\textwidth}{
%        \centering
%        \includegraphics[width=0.2\textwidth]{figures/sampleCov}    
%    }
%    \end{subfigure}
\begin{center}
    \begin{subfigure}[$\widetilde{\mbf{G}}$ of \emph{Movielens}]{
        \centering
        \includegraphics[width=0.27\textwidth]{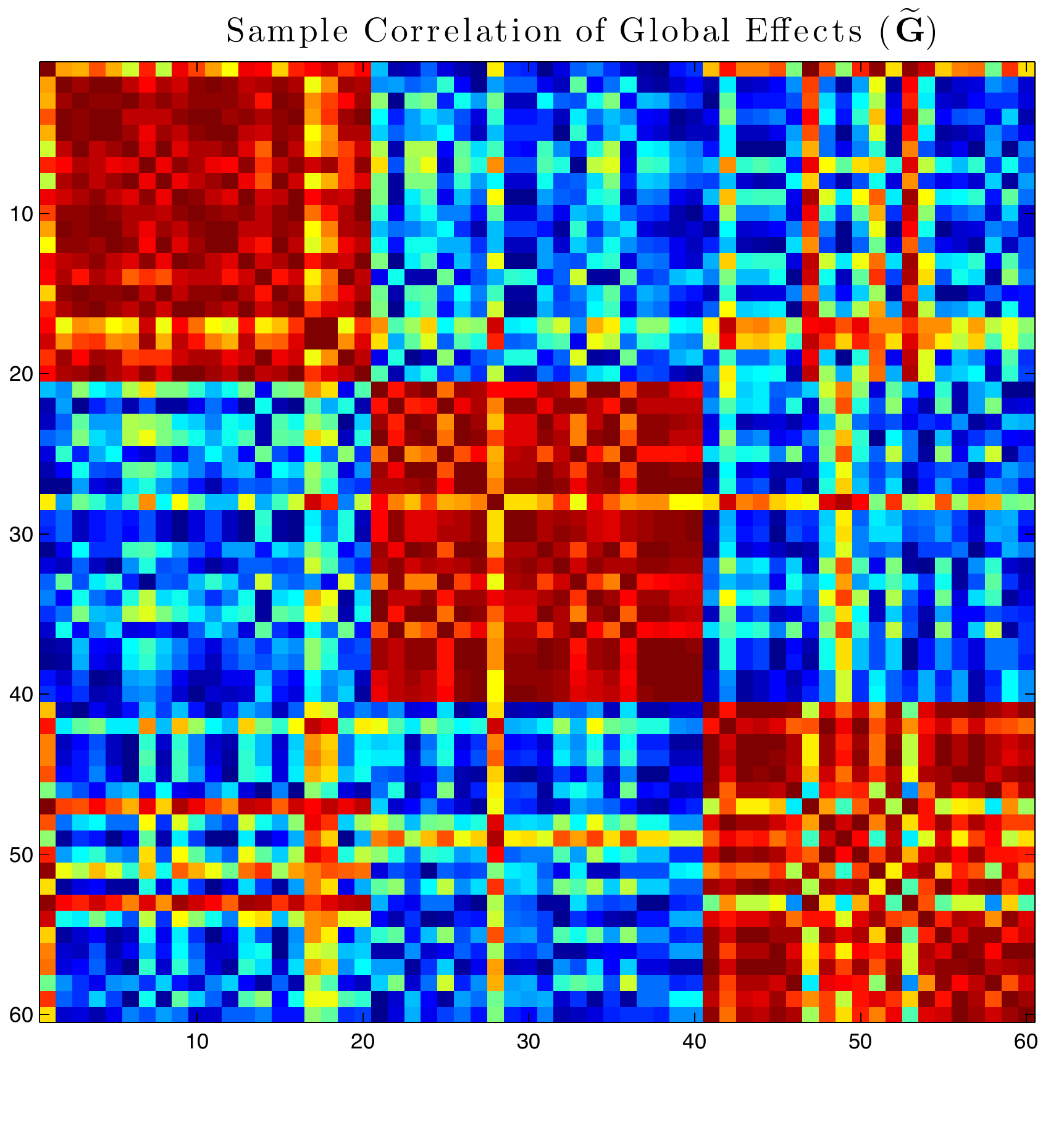}    
        \label{subfig:movie:G}
    }
    \end{subfigure}
    \begin{subfigure}[$|\widetilde{\mbf{S}}|$ and $|\widetilde{\mbf{S}}^{-1}|$ of \emph{Movielens}]{
        \centering
        \includegraphics[width=0.6\textwidth]{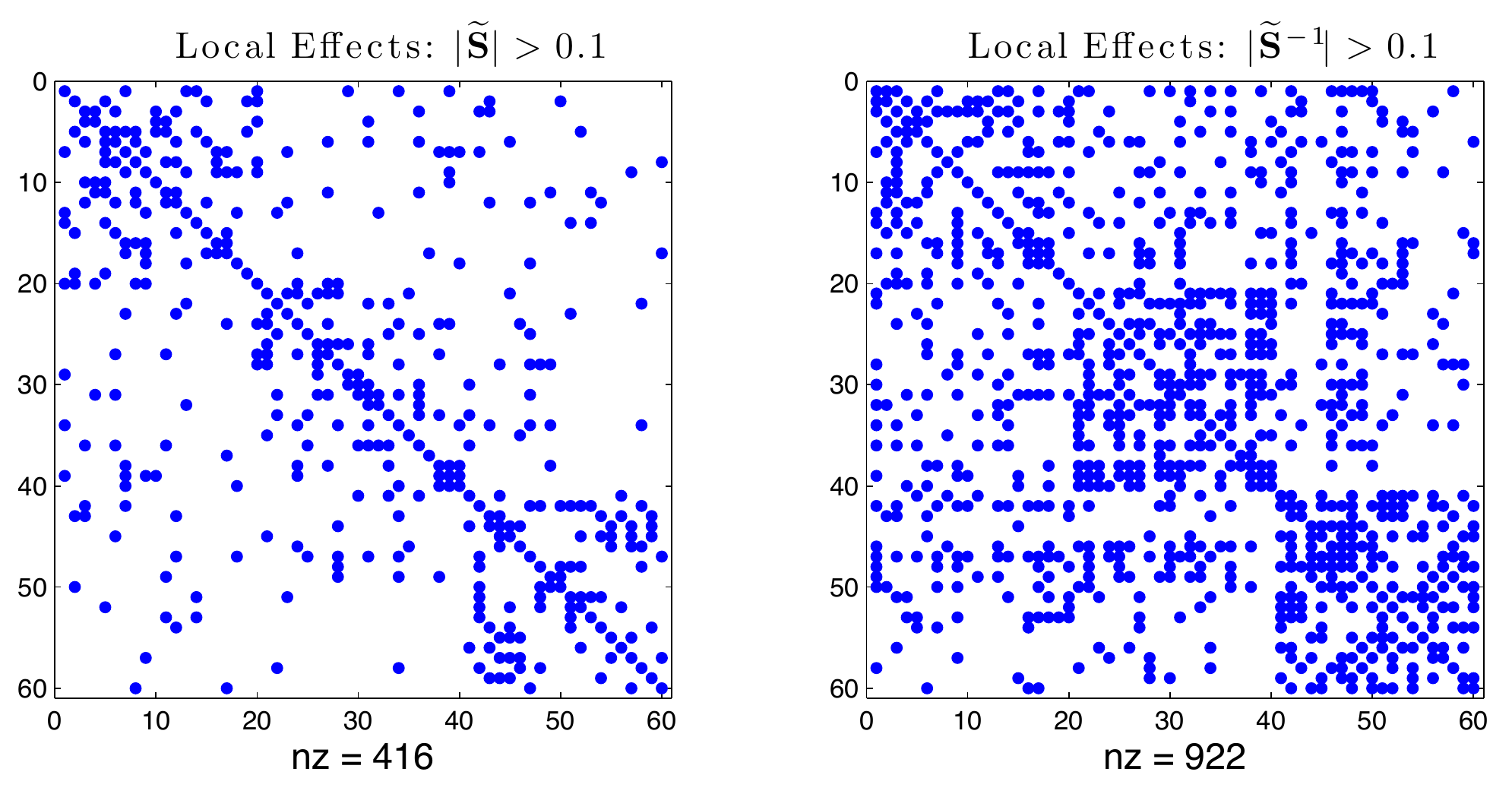}  
        \label{subfig:movie:S}
    }
    \end{subfigure}
    \\
    \begin{subfigure}[$\widetilde{\mbf{G}}$ of stock data]{
        \centering
        \includegraphics[width=0.27\textwidth]{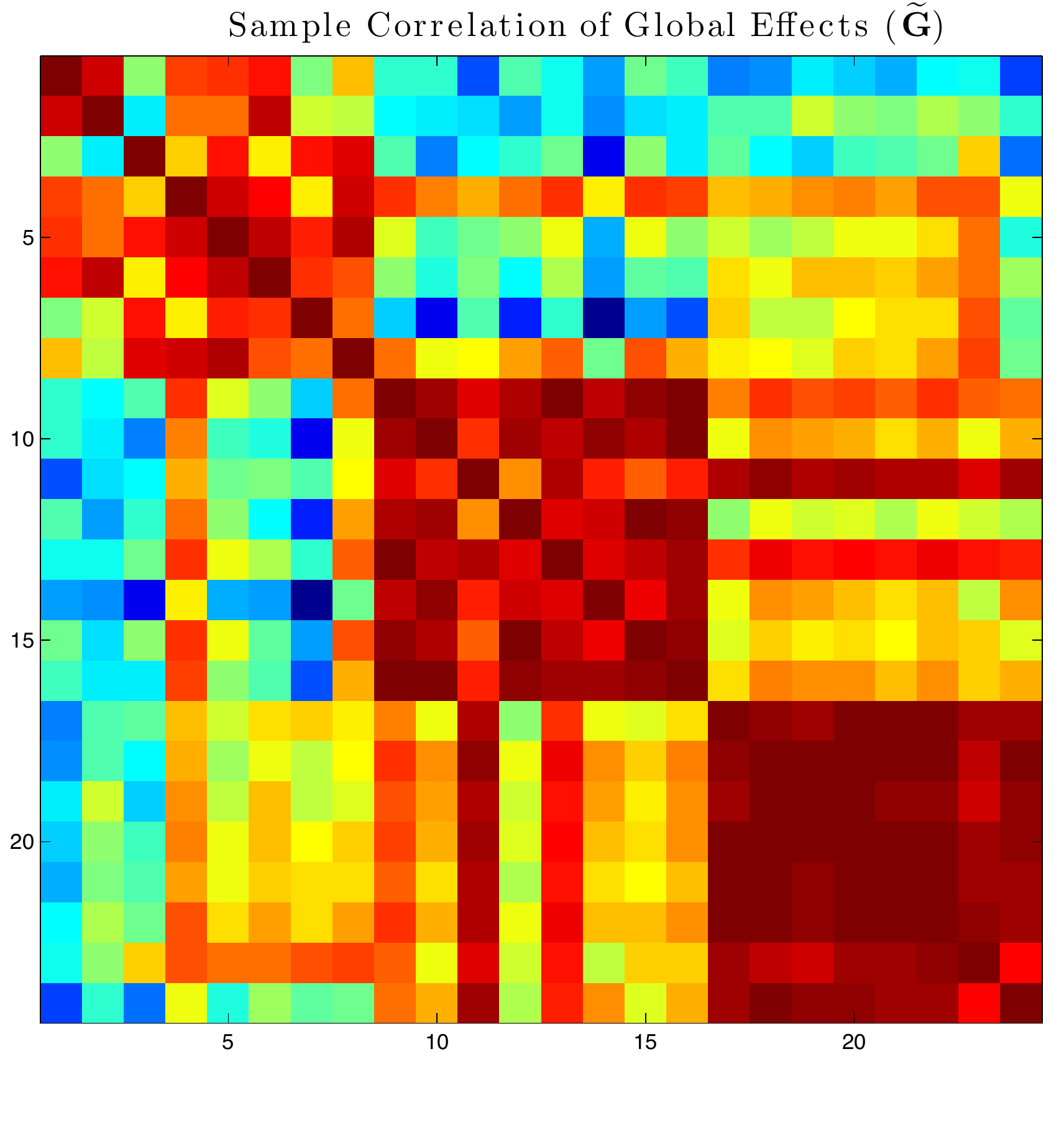}
        \label{subfig:stock:G}    
    }
    \end{subfigure}
    \begin{subfigure}[$|\widetilde{\mbf{S}}|$ and $|\widetilde{\mbf{S}}^{-1}|$ of stock data]{
        \centering
        \includegraphics[width=0.6\textwidth]{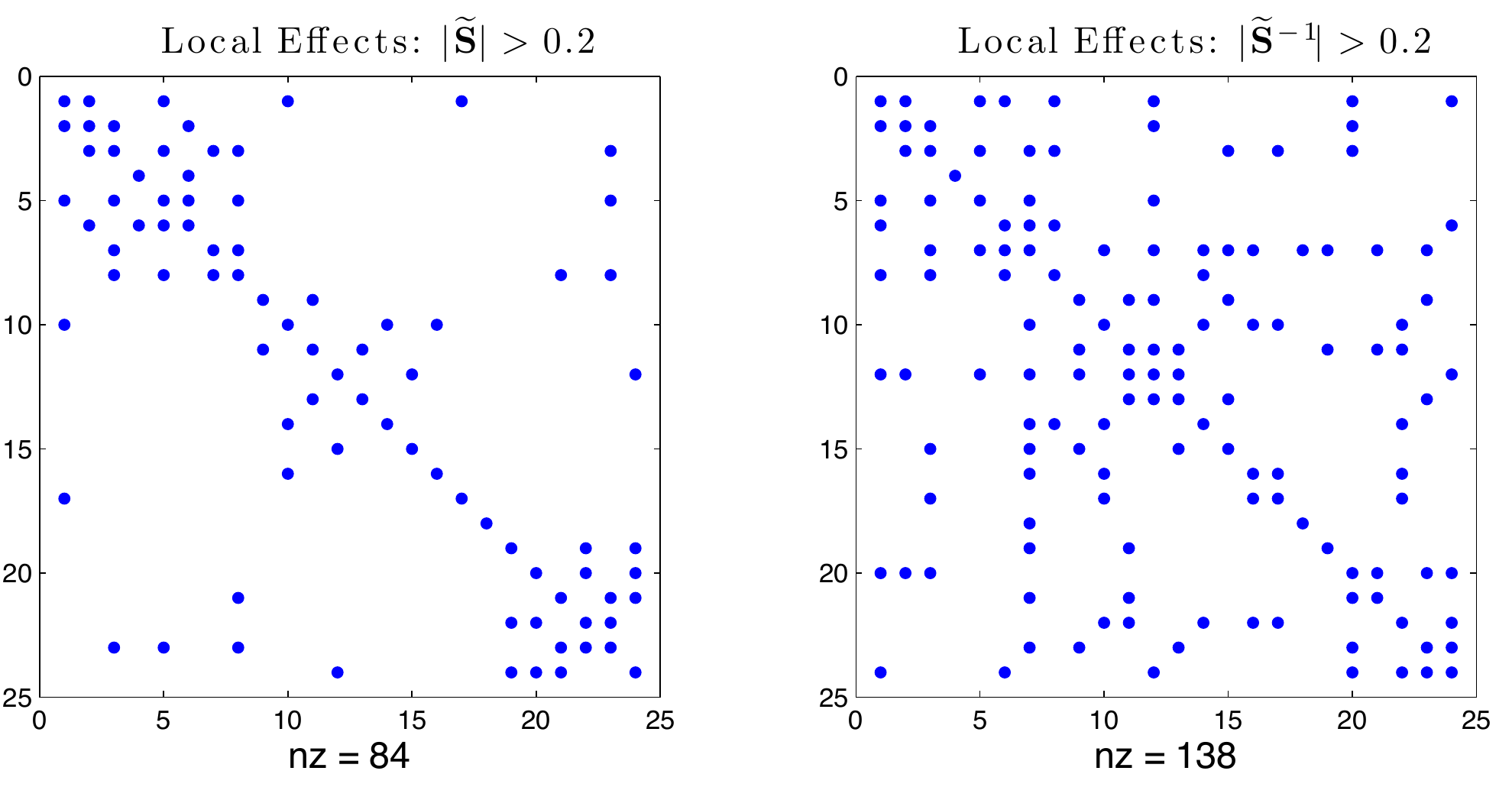}   
        \label{subfig:stock:S} 
    }
    \end{subfigure}
\end{center}
\caption{Illustration of LVGGM assumptions on \emph{Movielens} and stock return data sets. (a)(c): Heat maps of the leading low-rank matrices capturing the global effects. (b)(d): Sparsity patterns of the precision and covariance matrices of the remaining conditional effects.}
\end{figure*}

%\section{Proof of Theorem~\ref{thm:main}}
%\label{sec:proof:main}

\section{Proof of Theorem~\ref{thm:main}}
\label{sec:proof:main}

In~\citet{yang2013dirty}, the authors proved a general superpositioned parameter estimate error bound using the decomposable regularized framework. Theorem~\ref{thm:main} can be proven similarly by specializing the result in~\citet{yang2013dirty} to the LVGGM learning problem~\eqref{eq:MestLVGGM}. Then it suffices to verify the two critical conditions (C3) and (C4) in~\citet{yang2013dirty} (the other two conditions are trivial to verify for our problem), which we introduce and elaborate in this section. 

\paragraph{Restricted strong convexity.} 
%First we introduce the \emph{restricted strong convexity} (RSC) condition for the loss function. 
Let $\delta \mcal{L}({\Delta}; {\Theta}^*)$ denote the remainder term in first-order Taylor series approximation of the loss function $\mcal{L}(\cdot)$ at the true parameter $\Theta^{*}$ with respect to a perturbation $\Delta = \Theta^* - \widehat{\Theta}$:
\ALGNN{
\delta \mcal{L}({\Delta}; {\Theta}^*) := \mcal{L}(\Theta^* + \Delta) - \mcal{L}(\Theta^*) - \langle \nabla \mcal{L}(\Theta^*), \Delta \rangle.
}
In \citet{negahban2012unified}, the authors introduce the \emph{restricted strong convexity} (RSC) condition, which specifies that given some set $\mathbb{C} \subseteq \mathbb{R}^{p \times p}$, there exists some curvature parameter $\kappaL > 0$ and tolerance function $\tau_{\mcal{L}}$, such that the following holds:
\ALGNN{
\delta \mcal{L}({\Delta}; {\Theta}^*) \ge \kappaL \| {\Delta} \|^2_F - \tau_{\mcal{L}}(\Theta^*), \ \forall \Delta \in \mathbb{C}.
\label{eq:RSC}
}

The RSC condition guarantees sufficient curvature of the loss function at the true parameter along some directions specified by set $\mathbb{C}$. This condition is critical for consistent estimation in the high-dimensional regime, since standard strong convexity usually does not hold in the $p \gg n$ setting.

The following shows that the restricted Fisher eigenvalue conditions defined in Assumption~\ref{assump:rfe} implies the RSC condition.
 
\begin{lem}[RSC condition]	\label{lem:rsc}
Suppose Assumption~\ref{assump:rfe} holds for the true marginal precision matrix $\Thetastar$ and let $M > 2$.  Then for all $\Delta \in \mathbb{C}(E) \cup \mathbb{C}(U)$, such that $\| \Delta \|_{\mcal{F}^*}^2 \le \frac{1}{2 M^2}$, the RSC condition~\eqref{eq:RSC} is satisfied with the curvature parameter $\kappaL = \frac{M-2}{2(M-1)} \kappa_{\min}^*$ and the tolerance function $\tau_{\mcal{L}} = 0$.
\end{lem}

The proof of Lemma~\ref{lem:rsc} is largely inspired by~\citet{kakade2009learning}, in which it is shown that exponential family distributions exhibit \emph{almost strong convexity} in a neighborhood. The RFE assumption makes connection between this property and the RSC condition. A proof of Lemma~\ref{lem:rsc} is given in the Appendix~\ref{sec:proof:RSC}. 

Note there is an important difference between the RSC condition considered here and the condition introduced in~\citet{agarwal2012noisy}.  The RSC condition considered here is satisfied with respect to the error matrices of each simple structure separately, while the RSC condition in~\citet{agarwal2012noisy} is required for the combined error matrices (defined in the product space of two sets), which could lie in a significantly larger set.

\paragraph{Structural incoherence.} The second ingredient for consistent estimation of the sparse plus low-rank parameter $\Theta$, is some type of incoherence condition between the sparse and low-rank components. 
In the present work, we consider the \emph{structural incoherence} condition that was proposed more recently in~\cite{yang2013dirty}.  This condition allows for a vanishing error bound when $n$ goes to infinity, and is applicable to more general loss functions, such as the log-likelihood function in Eq.~\eqref{eq:llh}.

Define the following incoherence measure of the loss function $\mcal{L}$ for two structural error matrices $\Delta_S$ and $\Delta_L$:
%between the two sets $\mathbb{C}(E)$ and $\mathbb{C}(U)$:
\ALGN{
\begin{split}
& c_{\mcal{L}}(\Delta_S, \Delta_L; \Thetastar) := | \mcal{L}(\Thetastar + \Delta_S + \Delta_L) + \mcal{L}(\Thetastar) \\
& - \mcal{L}(\Thetastar + \Delta_S) - \mcal{L}(\Thetastar + \Delta_L) |, \forall \Delta_S \in \mathbb{C}(E), \Delta_L \in \mathbb{C}(U).
\end{split}
}

Then the \emph{structural incoherence} (SI) condition is satisfied if the following relation holds for all $\Delta_S \in \mathbb{C}(E)$ and $\Delta_L \in \mathbb{C}(U)$:
\ALGNN{
c_{\mcal{L}}(\Delta_S, \Delta_L; \Thetastar) \le  \frac{\kappaL}{2} \| ( \| \Delta_S \|_F^2 + \| \Delta_L \|_F^2 ),
\label{eq:si}
}
where $\kappaL$ is the curvature parameter in the RSC condition~\eqref{eq:RSC}.

The following lemma shows that, in addition to the restricted Fisher eigenvalue assumption (Assumption~\ref{assump:rfe}), if the true marginal model also satisfies the structural Fisher incoherence assumption (Assumption~\ref{assump:sfi}), then the above SI condition on the likelihood loss function is guaranteed. 

\begin{lem}[SI condition]    \label{lem:si}
Suppose Assumption~\ref{assump:rfe} and~\ref{assump:sfi} hold for the true marginal precision matrix $\Thetastar$ and let $M>6$. 
Then the SI condition~\eqref{eq:si} is satisfied for all $\Delta_S \in \mathbb{C}(E)$ and $\Delta_L \in \mathbb{C}(U)$, such that $\max \{ \| \Delta_S \|_{\mcal{F}^*}^2, \| \Delta_L \|_{\mcal{F}^*}^2 \} \le \frac{1}{6M^2}$. The curvature parameter $\kappaL$ is the same as in Lemma~\ref{lem:rsc}, i.e., $\kappaL = \frac{M-2}{2(M-1)} \kappa_{\min}^*$.
\end{lem}

The proof of Lemma~\ref{lem:si} is in Section~\ref{sec:proof:lem:si} in Appendix. 

%\begin{proof}
%(Theorem~\ref{thm:main}) 
Finally, under Assumption~\ref{assump:rfe} and Assumption~\ref{assump:sfi}, Lemma~\ref{lem:rsc} and~\ref{lem:si} imply the RSC and SI conditions hold for our LVGGM learning problem, respectively. Thus Theorem~\ref{thm:main} can be proven by directly appealing to Theorem 1 in~\cite{yang2013dirty}. 
%\end{proof}

\section{Proof of Lemma~\ref{lem:rsc}}
\label{sec:proof:RSC}

\begin{proof}
The remainder term in the first-order Taylor series of the negative log-likelihood~\eqref{eq:llh} of GGM takes the following form:
\ALGN{
\delta \mcal{L}({\Delta}; {\Theta}^*) & = \mcal{L}(\Theta^* + \Delta) - \mcal{L}(\Theta^*) - \langle \nabla \mcal{L}(\Theta^*), \Delta \rangle \\
& = \langle {\Thetastar}^{-1}, \Delta \rangle - \log \det(\Thetastar + \Delta) + \log \det(\Thetastar). 
}

For $s \in (0, 1]$, define the Taylor series of function $g(s; \Thetastar) : = \log \det (\Thetastar + s \Delta)$ at $\Thetastar$
\ALGNN{
g(s; \Thetastar) = \log \det (\Thetastar + s \Delta) = \sum_{k = 0}^{\infty} \frac{c_k (\Delta) s^k}{k!},
}
where $c_k(\Delta) := g^{(k)}(s; \Thetastar)$ is the $k$-th derivative of the $\log \det$ function at $\Thetastar$. Define $c_0(\Delta) : = \log \det(\Thetastar)$, the remainder can be expressed as:
\ALGNN{
\delta \mcal{L}(s {\Delta}; {\Theta}^*) = \sum_{k = 2}^{\infty} \frac{c_k (\Delta) s^k}{k!} = \frac{c_2(\Delta) s^2}{2} + \sum_{k = 3}^{\infty} \frac{c_k (\Delta) s^k}{k!} = \frac{c_2(\Delta) s^2}{2} + \delta g(s; \Delta, \Thetastar),    \label{eq:taylor:deltaL}
}
where the second term $\delta g(s)$ is defined as the second-order Taylor error of the log-determinant function. Next we show that this error term, which is the sum of all the higher-order terms, can be bounded by a quadratic term in a small neighborhood around $\Thetastar$. 

For exponential family distributions (Gaussian as an example), the log-partition function (\emph{i.e.,} $\log \det$ function for Gaussian) coincides with the \emph{cumulant generating function}. This implies that the derivatives $c_k(\Delta)$ are the corresponding cumulants of the distribution, which can be shown to converge to zero quite rapidly. Indeed, in~\citet{kakade2009learning} the authors show that for a univariate random variable $z$ under an exponential family distribution, its $k$-th order cumulant satisfies
\ALGNN{
\left| \frac{c_k (z)}{c_2 (z)^{k/2}} \right| \le \frac{1}{2} k! \alpha^{k-2}, \quad \forall k \ge 3,
}
where $\alpha$ is a finite constant, and the second-order cumulant coincides with the Fisher norm of the deviation $c_2(\Delta) = \| \Delta \|_{\mcal{F}^*}^2$ due to the definition of the Fisher information. For multivariate Gaussian distributions, $\alpha = \sqrt{2}$ suffices for the above relation to hold (see Sec.~3.2.2 in~\citet{kakade2009learning}). 

Therefore we bound the second-order Taylor error term in Eq.~\eqref{eq:taylor:deltaL} as follows (similar to~\citet{kakade2009learning}):
\ALGNN{
\left| \delta g(s; \Delta, \Thetastar) \right| & = \left| \sum_{k = 3}^{\infty} \frac{c_k (\Delta) s^k}{k!} \right| \\
& \le \frac{1}{2} \sum_{k = 3}^{\infty} 2^{\frac{k}{2}-1} c_2(\Delta)^{k/2} s^k \\
& \le \frac{s^2 c_2(\Delta)}{2} \sum_{k = 1}^{\infty} (s \sqrt{2 c_2 (\Delta)})^k \\
& \overset{(i)}{\le} \frac{s^2 c_2(\Delta)}{2} \sum_{k = 1}^{\infty} \frac{1}{M^k} \\
& = \frac{s^2 c_2(\Delta)}{2 (M-1)} \\
& \le \frac{c_2 (\Delta) }{2(M-1)} \frac{1}{\max \{ 2M^2 c_2 (\Delta), 1\} } \\
& \overset{(ii)}{=} \frac{c_2 (\Delta) }{2(M-1)}     \label{eq:taylor:deltag}
}
where $(i)$ and $(ii)$ are due to our conditions on $c_2 (\Delta)$ (\emph{i.e.,} $\| \Delta \|_{\mcal{F}^*}^2 \le \frac{1}{2 M^2}$) and $s \le 1$. Then we obtain a lower bound for $\delta \mcal{L}({\Delta}; {\Theta}^*)$:
\ALGNN{
\delta \mcal{L}({\Delta}; {\Theta}^*) \ge \frac{c_2(\Delta)}{2} + \delta g(s; \Delta, \Thetastar) \ge \left( \frac{1}{2} - \frac{1 }{2(M-1)} \right) c_2 (\Delta) \overset{(ii)}{\ge} \frac{M-2}{2(M-1)} \kappa_{\min}^* \| \Delta \|_F^2,
}
where $(ii)$ is due to the RFE condition.
Therefore the RSC condition is satisfied with the curvature parameter $\kappaL := \frac{M-2}{2(M-1)} \kappa_{\min}^*$ and a zero tolerance parameter $\tauL = 0$.

\end{proof}

\section{Proof of Lemma~\ref{lem:si}}
\label{sec:proof:lem:si}

\begin{proof}

First we state the following lemma which gives a bound on the magnitude of \emph{Fisher inner product} between elements from the two sets.
\begin{lem}    \label{lem:fisher:inner}
Suppose Assumption~\ref{assump:rfe} and~\ref{assump:sfi} hold for the true marginal precision matrix $\Thetastar$. 
Then given a constant $M \ge 6$, the following inequality holds for all $\Delta_S \in \mathbb{C}(E)$ and $\Delta_L \in \mathbb{C}(U)$ such that $\max \{ \| \Delta_S \|_{\mcal{F}^*}^2, \| \Delta_L \|_{\mcal{F}^*}^2 \} \le \frac{1}{6 M^2}$:
\ALGNN{
\left| \langle \Delta_S, \Delta_L \rangle_{\mcal{F}^*} \right| \le \psi \left( \| \Delta_S \|_{\mcal{F}^*}^2 + \| \Delta_L \|_{\mcal{F}^*}^2 \right),
}
where $\psi := \frac{1}{4} - \frac{3}{2M}$. 
\end{lem}

The proof of Lemma~\ref{lem:fisher:inner} follows similarly as that of the Proposition 2 in~\citet{yang2013dirty}, and hence is omitted. 

Next we prove Lemma~\ref{lem:si} using the above result. 
Following similar derivations as in the proof of Lemma~\ref{lem:rsc}, the incoherence measure in the SI condition can be simplified to
\ALGN{
%& \left| \mcal{L}(\Thetastar + \Delta_S + \Delta_L) + \mcal{L}(\Thetastar) - \mcal{L}(\Thetastar + \Delta_S) - \mcal{L}(\Thetastar + \Delta_L) \right| \\
\begin{split}
c_{\mcal{L}}(\Delta_S, \Delta_L; \Thetastar) := \left| \delta \mcal{L}(\Delta_S + \Delta_L; \Thetastar) - \delta \mcal{L}(\Delta_S; \Thetastar) - \delta \mcal{L}(\Delta_L; \Thetastar) \right|.
\end{split}
}

Using the remainder in the Taylor series of $\delta \mcal{L}$~\eqref{eq:taylor:deltaL}, the incoherence measure can be expressed as:
\ALGN{
& c_{\mcal{L}}(\Delta_S, \Delta_L; \Thetastar) \\
= & \left| \frac{c_2(\Delta_S + \Delta_L)}{2} + \delta g(s; \Delta_S+\Delta_L) - \left( \frac{c_2(\Delta_S)}{2} + \delta g(s_1; \Delta_S) \right) - \left( \frac{c_2(\Delta_L)}{2} + \delta g(s_2; \Delta_L) \right) \right| \\
\le & \left| \frac{c_2(\Delta_S + \Delta_L)}{2} - \frac{c_2(\Delta_S)}{2} -  \frac{c_2(\Delta_L)}{2} \right| + \left| \delta g(s; \Delta_S+\Delta_L) \right| + \left| \delta g(s_1; \Delta_S) \right| + \left| \delta g(s_2; \Delta_L) \right| \\
\overset{(i)}{\le} & | \langle \Delta_S, \Delta_L \rangle_{\mcal{F}^*} | + \frac{c_2 (\Delta_S + \Delta_L) + c_2 (\Delta_S) + c_2 (\Delta_L) }{2(M-1)} \\
= & | \langle \Delta_S, \Delta_L \rangle_{\mcal{F}^*} | + \frac{\| \Delta_S \|_{\mcal{F}^*}^2 + \| \Delta_L \|_{\mcal{F}^*}^2 + \langle \Delta_S, \Delta_L \rangle_{\mcal{F}^*} }{M-1} \\
\le & \frac{M}{M-1} | \langle \Delta_S, \Delta_L \rangle_{\mcal{F}^*} | + \frac{\| \Delta_S \|_{\mcal{F}^*}^2 + \| \Delta_L \|_{\mcal{F}^*}^2}{M-1} \\
\overset{(ii)}{\le} & \frac{M \psi + 1}{M-1} (\| \Delta_S \|_{\mcal{F}^*}^2 + \| \Delta_L \|_{\mcal{F}^*}^2 ) \\
\overset{(iii)}{\le} & \frac{M - 2}{4(M-1)} {\kappa_{\min}^*} (\| \Delta_S \|_F^2 + \| \Delta_L \|_F^2 ) \\
\le & \frac{\kappaL}{2} (\| \Delta_S \|_{F}^2 + \| \Delta_L \|_{F}^2 ),
}
where in $(i)$ we have apply~\eqref{eq:taylor:deltag} to bound the second-order Taylor error terms (note that the conditions on the error matrices also guarantees $\| \Delta_S + \Delta_L \|_{\mcal{F}^*}^2 \le \frac{1}{2M^2}$ due to Lemma~\ref{lem:fisher:inner}).  Inequality $(ii)$ is due to Lemma~\ref{lem:fisher:inner}. Inequality $(iii)$ can be verified by the definitions of $\psi$ and the RSC curvature parameter $\kappaL$. 
\end{proof}

\section{Proof of Corollary~\ref{cor:main}}
\label{sec:proof:cor:main}

\begin{proof}

Theorem~\ref{thm:main} is a deterministic statement, however, the condition on the regularization parameters~\eqref{eq:regparams} and the error bound depend on the sample covariance matrix $\widehat{\Sigma}$ which is random. Note that the error bound directly follows from the deterministic error bound in Theorem~\ref{thm:main} and the choices of regularization parameters as in Eq.~\eqref{eq:regcond:cor1}. To prove Corollary~\ref{cor:main}, it only remains to verify that the condition~\eqref{eq:regparams} in Theorem~\ref{thm:main} is guaranteed with high probability. More specifically, this requires bounding the deviation of the sample covariance matrix in terms of $\ell_\infty$ and and spectral norms. 

First we make use of the following lemma to characterize the element-wise deviation of the sample covariance matrix\footnote{The original lemma applies to all sub-Gaussian variables, here we specialize to Gaussian random vectors.}.
\begin{lem}[\citet{ravikumar2011high}]
For a $p$-dimensional Gaussian random vector with covariance matrix $\Sigstar$, the sample covariance matrix obtained from $n$ samples $\Sighat$ satisfies
\ALGNN{
P\left\{ | \Sighat_{i,j} - \Sigstar_{i,j} | > \epsilon_{1} \right\} \le 4 \exp \left( - \frac{n  \epsilon_{1}^{2}}{3200 {\overline{\sigma}^*}^{2}} \right),
}
for all $\epsilon_{1} \in (0, 40 \overline{\sigma})$, where 
%\ALGNN{
$\overline{\sigma}^* := \max_{i=1,\ldots,p} \Sigstar_{i,i}$.
%}
\label{lem:sampledeviation:inf}
\end{lem}

If the number of samples satisfies $n \ge 4 \log p$, then by choosing $\frac{1}{2} \la \ge \epsilon_{1} = 80 C_1 \overline{\sigma}^* \sqrt{\frac{\log p^{2}}{n}} \in (0, 40 \overline{\sigma})$, where $C_1 > 1$ is an arbitrary constant, and applying the union bound we have
\ALGN{
P\left\{ \| \Sighat - \Sigstar \|_{\infty} \le \frac{1}{2} \la \right\} \ge P \left\{ \| \Sighat - \Sigstar \|_{\infty} \le \epsilon_{1} \right\} \ge 1 - 4 p^{-2(C_1 -1)}.
}
Then the condition on $\la$ is satisfied with high probability.  

Next we consider the condition on the other regularization parameter $\mu$, which requires bounding the deviation of the operation norm of the sample covariance matrix. The following lemma provides such a characterization.

\begin{lem}[\citet{chandrasekaran2012latent}, Lemma 3.9]
For a $p$-dimension Gaussian random vector with covariance matrix $\Sigstar$ and let $\rho^* = \| \Sigstar \|_{2}$. If the number of samples $n$ be such that $n \ge \frac{64 p {\rho^*}^{2} }{\epsilon_{2}^{2}}$, then the sample covariance matrix $\Sighat$ obtained from $n$ samples satisfies
\ALGNN{
P\left\{ \| \Sighat - \Sigstar \|_{2} \ge \epsilon_{2} \right\} \le 2 \exp \left( - \frac{n \epsilon_{2}^{2}}{128 {\rho^*}^{2} } \right),
}
for all $\epsilon_{2} \in (0, 8 \rho^*)$.
\end{lem}

If $n \ge p$, then by choosing $\frac{1}{2} \mu \ge \epsilon_{2} = 8 C_2 \rho^* \sqrt{\frac{p}{n}} \in (0, 8 \rho^*)$, where $C_2 \ge 1$ is an arbitrary constant, we have
\ALGN{
P \left\{ \| \Sighat - \Sigstar \|_{2} \le \frac{1}{2} \mu \right\} \ge P\left\{ \| \Sighat - \Sigstar \|_{2} \le \epsilon_{2} \right\} \ge 1 - 2 \exp \left( - \frac{C_2^2 p}{2} \right).
}

Combining the above results we have verified the condition~\eqref{eq:regparams} in Theorem~\ref{thm:main} holds with high probability, which concludes the proof.
\end{proof}

\end{document}